\documentclass[journal,10pt]{IEEEtran}

\usepackage{tabularx}
\usepackage{array}

\usepackage{amsmath,amsfonts,bm,bbm}
\allowdisplaybreaks
\usepackage{color}
\usepackage{amsthm}
\usepackage{graphicx}
\usepackage{mathtools}
\usepackage{url}
\usepackage{algorithm,algorithmic,multirow,hhline}
\usepackage{dblfloatfix}
\usepackage{amssymb}
\usepackage{caption}
\usepackage{comment}
\usepackage{subfigure}
\usepackage{pstricks}
\usepackage{subfloat}
\usepackage{setspace}
\usepackage{hyperref}
\usepackage{color}
\usepackage{booktabs,multirow}
\usepackage{multicol}

\usepackage{cite}

\def\ie{\emph{i.e.,\ }}

\def\nm{\normalsize}

\providecommand{\norm}[1]{\left\|#1\right\|}

\usepackage{xpatch}
\xpatchcmd\algorithmic
  {\newcommand{\IF}}
  {%
    \newcommand\SCOPE{\begin{ALC@g}}%
    \newcommand\ENDSCOPE{\end{ALC@g}}%
    \newcommand{\IF}%
  }
  {}{\fail}

\usepackage{dblfloatfix}

\makeatletter
\newtheorem{theorem}{Theorem}
\newtheorem{lemma}{Lemma}

\title{\fontsize{23}{24}\selectfont Communication Efficient Distributed Learning with Censored, Quantized, and Generalized Group ADMM}

\author{Chaouki Ben Issaid, Anis Elgabli, $^\dagger$Jihong Park, Mehdi Bennis and $^\ddagger$M\'erouane Debbah \thanks{C. Ben Issaid, A. Elgabli and M. Bennis are with the Centre of Wireless Communications, University of Oulu, 90014 Oulu, Finland (email: \{chaouki.benissaid, anis.elgabli, mehdi.bennis\}@oulu.fi).}\\ 
\thanks{$^\dagger$J. Park is with the School of Information Technology, Deakin University, Geelong, VIC 3220, Australia (email: jihong.park@deakin.edu.au).}
\thanks{$^\ddagger$M. Debbah is with Universit\'e Paris-Saclay, CNRS, CentraleSup\'elec, 91190, Gif-sur-Yvette, France (e-mail: merouane.debbah@centralesupelec.fr) and the Lagrange Mathematical and Computing Research Center, 75007, France.
}
}

\begin{document}

\maketitle
\begin{abstract}
In this paper, we propose a communication-efficiently decentralized machine learning framework that solves a consensus optimization problem defined over a network of inter-connected workers. The proposed algorithm, \emph{Censored and Quantized Generalized GADMM} (CQ-GGADMM), leverages the worker grouping and decentralized learning ideas of \emph{Group Alternating Direction Method of Multipliers} (GADMM), and pushes the frontier in communication efficiency by extending its applicability to generalized network topologies, while incorporating link censoring for negligible updates after quantization. We theoretically prove that CQ-GGADMM achieves the linear convergence rate when the local objective functions are strongly convex under some mild assumptions. Numerical simulations corroborate that CQ-GGADMM exhibits higher communication efficiency in terms of the number of communication rounds and transmit energy consumption without compromising the accuracy and convergence speed, compared to the censored decentralized ADMM, and the worker grouping method of GADMM.
\end{abstract}


\section{Introduction}\label{introduction}


    Machine learning is central to emerging mission-critical applications such as autonomous driving, remote surgery, and the fifth-generation (5G) communication systems and beyond \cite{park2018wireless,UO:6G,zhu2020toward}. These applications commonly require extremely low latency and high reliability while accurately reacting to local environmental dynamics \cite{park2020extreme}. To this end, training their machine learning models needs the sheer amount of fresh training data samples that are generated by and dispersed across edge devices (e.g., phones, cars, access points, etc.), hereafter referred to as workers. Collecting these raw data may not only violate the data privacy, but also incur significant communication overhead under limited bandwidth. This calls for developing communication-efficient and privacy-preserving distributed learning frameworks \cite{park2020:cml,Chen2017MachineLF}. Federated learning is one representative method that ensures learning through periodically exchanging model parameters across workers rather than sending private data samples \cite{kairouz2019advances}. Nevertheless, federated learning postulates a parameter server collecting and distributing model parameters, which is not always accessible from faraway workers and is vulnerable to a single point of failure \cite{KimCL:19}. 

\begin{figure*}
    \centering
    \includegraphics[width=\textwidth]{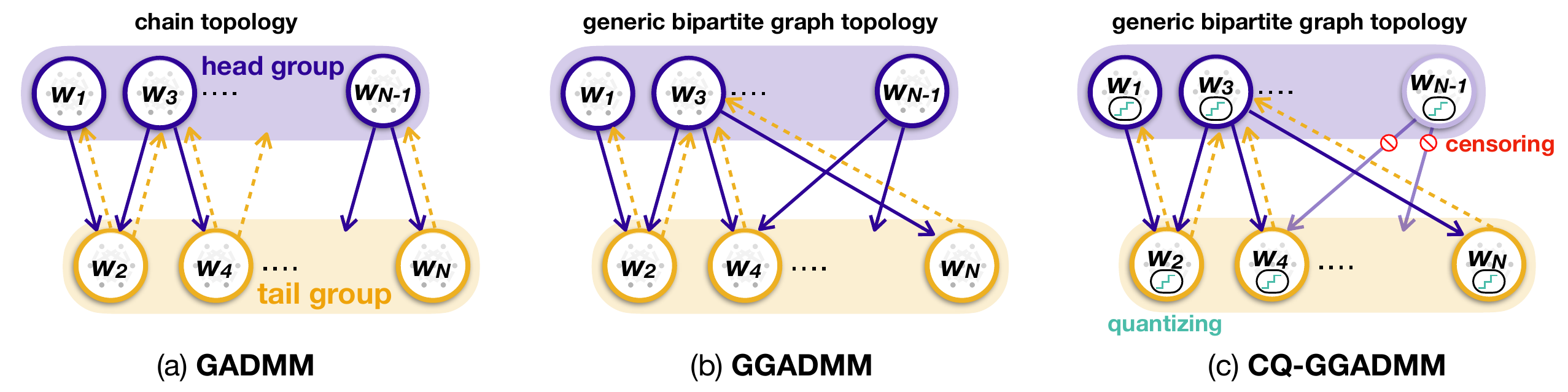}
    \caption{A schematic illustration of (a) \emph{group ADMM (GADMM)}, (b) \emph{generalized GADMM (GGADMM)} and (c) \emph{censored and quantized GGADMM (CQ-GGADMM)}.}
    \label{Fig_Overview} \vspace{-0.75cm}
    \end{figure*}

Spurred by this motivation, by generalizing and extending the Group Alternating Direction Method of Multipliers (GADMM, see Fig.~\ref{Fig_Overview}(a)) and the Quantized GADMM (Q-GADMM) in our prior work \cite{anis2020, anis2019}, in this article we propose a novel decentralized learning framework, coined \emph{Censored and Quantized Generalized Group ADMM} (CQ-GGADMM, see Fig.~\ref{Fig_Overview}(c)), which exchanges model parameters in a communication-efficient way without any central entity. Following the same idea of GADMM, workers in CQ-GGADMM are divided into head and tail groups in which the workers in the same group update their models in parallel, whereas the workers in different groups update their models in an alternating way. In essence, CQ-GGADMM exploits two key principles to improve the communication efficiency. First, to reduce the number of communication links per round, CQ-GGADMM exploits a censoring approach that allows to exchange model parameters only when the updated model is sufficiently changed from the previous model, i.e., skipping small model updates \cite{liu2019}. Second, to reduce the communication payload size per each link, CQ-GGADMM applies a heterogeneous stochastic quantization scheme that decreases the number of bits to represent each model parameter \cite{anis2019}. These principles are integrated giving rise to a generalized version of GADMM (GGADMM, see Fig.~\ref{Fig_Overview}(b)) wherein each worker communicates only with its neighboring workers. Note that in the original GADMM, every worker needs to connect with two neighbors under a chain network topology \cite{anis2020}. By contrast, in CQ-GGADMM, each worker can connect with an arbitrary number of neighbors, as long as the network topology graph is bipartite and connected.






Towards improving the communication efficiency of distributed learning, prior works have studied various techniques under centralized and decentralized network architectures.
\paragraph{Fast Convergence}
The total communication cost until completing a distributed learning operation can be reduced by accelerating the convergence speed. To this end, departing from the conventional first-order methods such as distributed gradient descent~\cite{boyd2011distributed}, second-order methods are applied under centralized \cite{Konecny2016, liu2019, elgabli2020harnessing} and decentralized architectures \cite{anis2020}. Furthermore, momentum acceleration is utilized under centralized \cite{yu2019linear,gitman2019understanding} and decentralized settings \cite{gao2020adaptive}.

\paragraph{Link Sparsification}
In large-scale distributed learning, a large portion of total communication links is often redundant \cite{mishchenko20a}. In this respect, for each communication round, sparsifying the number of communication links can reduce the communication cost without compromising the accuracy. To this end, link censoring for negligible model updates is applied under centralized \cite{chen2018} and decentralized network topologies \cite{singh2019sparqsgd,anis2020}.

\paragraph{Payload Size Reduction}
To reduce the communication payload size per link, model updates are quantized under centralized \cite{bernstein18a, alistarh2017,amiri2020federated} and decentralized network topologies \cite{nandan2019,gao2020adaptive,anis2019}. Alternatively, the entries of model updates can be partially dropped as shown under centralized \cite{wangni2018} and decentralized architectures \cite{Stich:18}. Furthermore, under centralized settings, model parameters can be compressed at the parameter server via knowledge distillation (KD) or while training and running KD simultaneously, i.e., federated distillation \cite{Jeong18,Ahn}.

Although the aforementioned principles have been separately studied in \cite{anis2020, liu2019 ,anis2019}, integrating them for maximizing the communication efficiency while guaranteeing fast convergence remains a non-trivial problem. Indeed, first the algorithm convergence rate depends highly on the network topology. Second, both censoring and quantization steps incur model update errors that may propagate over communication rounds due to the lack of central entity. To resolve this problem, we carefully determine the non-increasing target censoring threshold and quantization step size, such that the model updates are more finely tuned as time elapses until convergence. We thereby prove the linear convergence rate of CQ-GGADMM, and show its effectiveness by simulations, in terms of convergence speed, total communication cost, and transmission energy consumption. The major contributions of this work are summarized as follows.
\begin{itemize}
    \item We propose CQ-GGADMM, a second-order decentralized learning framework utilizing censoring, quantization, and GADMM for any bipartite and connected network topology graph (\textbf{Algorithm~\ref{algo2}} in Sec.~\ref{CQGADMM}).
    \item We prove that CQ-GGADMM converges to the optimal solution for convex loss functions (\textbf{Theorem}~\ref{thm1} in Sec.~\ref{convergence}). 
    \item We identify the network topology conditions under which CQ-GGADMM achieves a linear convergence rate (\textbf{Theorem}~\ref{thm2} in Sec.~\ref{convergence}) when the loss functions are strongly convex. 
    \item Numerical simulations have corroborated that in linear and logistic regression tasks using synthetic and real datasets, CQ-GGADMM achieves the same convergence speed at significantly lower communication rounds and several orders of magnitude less transmission energy, compared to C-GGADMM and Censored ADMM (C-ADMM) in \cite{liu2019}.
    

\end{itemize}


The remainder of this paper is organized as follows. In section~II, we describe the generalized version of GADMM (GGADMM) for a bipartite and connected graph, and formulate the decentralized learning problem. Then, we extend GGADMM to censored and quantized GGADMM (CQ-GGADMM) in Section~III. In Section~IV, we prove the convergence of CQ-GGADMM theoretically under some mild conditions. Finally, Section~V validates the performance of CQ-GGADMM by simulations. The details of the proofs of our results are deferred to the appendices.
\section{Problem Formulation}
\label{probForm} 
We consider a connected network wherein a set $\mathcal{V}$ of $N$ workers aim to reach a consensus around a solution of a global optimization problem. The problem is solved using only local data and information available for each worker. Moreover,  communication is constrained to only take place between neighboring workers. The optimization problem is given by
\begin{align}
\textbf{(P1)} ~~ &\bm{\Theta}^* := \arg\min_{\bm{\Theta}} \sum_{n=1}^N f_n(\bm{\Theta}),
\end{align}
where $\bm{\Theta} \in  \mathbb{R}^{d \times 1}$ is the global model parameter and $f_n: \mathbb{R}^d \rightarrow \mathbb{R}$ is a local function composed of data stored at worker $n$. Problem $\textbf{(P1)}$ appears in many applications of machine learning, especially when the dataset is very large and the training is carried out using different workers. The connections among workers are represented as an undirected communication graph $\mathcal{G}$ having the set $\mathcal{E} \subseteq \mathcal{V} \times \mathcal{V}$ of edges. The set of neighbors of worker $n$ is defined as $\mathcal{N}_n = \{ m | (n, m) \in \mathcal{E} \}$ whose cardinality is $|\mathcal{N}_n| = d_{n}$. We start by making the following key assumption.

\textbf{Assumption 1.} The communication graph $\mathcal{G}$ is bipartite and connected.\\
Under \textbf{Assumption 1}, following the worker grouping of GADMM \cite{anis2020}, workers are divided into two groups: a \emph{head group} $\mathcal{H}$, and a \emph{tail group} $\mathcal{T}$. Each head worker in $\mathcal{H}$ can only communicate with tail workers in $\mathcal{T}$, and vice versa. In this case, the edge set definition can be re-written as $\mathcal{E} = \{ (n, m) | n \in \mathcal{H}, m \in \mathcal{T}\}$, and the problem \textbf{(P1)} is equivalent to the following problem
\begin{align}
\textbf{(P2)} ~~ &\bm{\theta}^* := \arg\min_{\{\bm{\theta}_n\}_{n=1}^N} \sum_{n=1}^N f_n(\bm{\theta}_n) \\ \nonumber &\text{s.t. }  \bm{\theta}_{n} = \bm{\theta}_{m}, \forall (n, m) \in \mathcal{E},
\end{align}
where $\bm{\theta}_n$ is the local copy of the common optimization variable $\bm{\Theta}$ at worker $n$. Note that, under the formulation \textbf{(P2)}, the objective function becomes separable across the workers and as a consequence the problem can be solved in a distributed manner. In this case, the Lagrangian of the optimization problem \textbf{(P2)} can be written as
\begin{align}
\nonumber \bm{\mathcal{L}}_\rho(\bm{\theta},\bm{\lambda}) &= \sum_{n=1}^N f_n(\bm{\theta}_n) + \sum_{(n,m) \in \mathcal{E}} \langle \bm{\lambda}_{n,m}, \bm{\theta}_{n} - \bm{\theta}_{m} \rangle\\
& + \frac{\rho}{2} \sum_{(n,m) \in \mathcal{E}} \|\bm{\theta}_{n} - \bm{\theta}_{m}\|^2,
\end{align}
where $\rho > 0$ is a constant penalty parameter and $\bm{\lambda}_{n,m}$ is the dual variable between neighboring workers $n$ and $m$, $\forall (n,m) \in \mathcal{E}$. At iteration $k+1$, the Generalized Group ADMM (GGADMM) algorithm runs as follows.
\begin{itemize}
\item[$(1)$] Every head worker, $n \in \mathcal{H}$, updates its primal variable by solving
\begin{align}
\nonumber \bm{\theta}_{n}^{k+1} &= \underset{\bm{\theta}_{n}}{\arg \min}~ f_n(\bm{\theta}_{n}) + \sum_{m \in \mathcal{N}_n} \langle \bm{\lambda}_{n,m}^{k} , \bm{\theta}_{n}-\bm{\theta}_{m}^{k}\rangle \\
&+ \frac{\rho}{2} \sum_{m \in \mathcal{N}_n} \|\bm{\theta}_{n}-\bm{\theta}_{m}^{k}\|^2,
\end{align}
and sends its updated model to its neighbors.
\item[$(2)$] The primal variables of tail workers, $m \in \mathcal{T}$, are then updated as
\begin{align}
\nonumber \bm{\theta}_{m}^{k+1} &= \underset{\bm{\theta}_{m}}{\arg \min}~ f_m(\bm{\theta}_{m}) + \sum_{n \in \mathcal{N}_m} \langle \bm{\lambda}_{n,m}^{k}, \bm{\theta}_{n}^{k+1}-\bm{\theta}_{m} \rangle\\
&+ \frac{\rho}{2} \sum_{n \in \mathcal{N}_m} \|\bm{\theta}_{n}^{k+1}-\bm{\theta}_{m}\|^2.
\end{align}
\item[$(3)$] The dual variables are updated locally for every worker, after receiving the model updates from its neighbors, in the following way
\begin{align}
\bm{\lambda}_{n,m}^{k+1} = \bm{\lambda}_{n,m}^{k} + \rho (\bm{\theta}_{n}^{k+1}-\bm{\theta}_{m}^{k+1}),~\forall (n, m) \in \mathcal{E}.
\end{align}
\end{itemize}
Note that GGADMM is a generalized version of GADMM algorithm proposed in \cite{anis2020} since it considers an arbitrary topology. Introducing $\bm{\alpha}_n = \sum_{m \in \mathcal{N}_n} \bm{\lambda}_{n,m}, ~\forall n \in \mathcal{V}$, we can write 
\begin{itemize}
\item[$(1)$] The update of the models of head workers is done in parallel by solving 
\begin{align}
\nonumber \bm{\theta}_{n}^{k+1} &= \underset{\bm{\theta}_{n}}{\arg \min}~ f_n(\bm{\theta}_{n}) + \langle \bm{\theta}_{n} , \bm{\alpha}_{n}^k- \rho \sum_{m \in \mathcal{N}_n} \bm{\theta}_{m}^{k}\rangle\\
& + \frac{\rho}{2} d_{n} \|\bm{\theta}_{n}\|^2.
\end{align}
\item[$(2)$] The models of tail workers are updated in parallel using
\begin{align}
\nonumber \bm{\theta}_{m}^{k+1} &= \underset{\bm{\theta}_{m}}{\arg \min}~ f_m(\bm{\theta}_{m}) + \langle \bm{\theta}_{m} , \bm{\alpha}_{m}^k- \rho \sum_{n \in \mathcal{N}_m} \bm{\theta}_{n}^{k+1}\rangle\\
& + \frac{\rho}{2} d_{m} \|\bm{\theta}_{m}\|^2.
\end{align}
\item[$(3)$] Instead of updating $\bm{\lambda}_{n,m}$, each worker will update locally the new auxiliary variable $\bm{\alpha}_n$ as
\begin{align}
\bm{\alpha}_n^{k+1} = \bm{\alpha}_n^{k} + \rho \sum_{m \in \mathcal{N}_n} (\bm{\theta}_{n}^{k+1} - \bm{\theta}_{m}^{k+1}), ~\forall n \in \mathcal{V}.
\end{align}
\end{itemize}

\section{Censored Quantized Generalized Group ADMM}\label{CQGADMM} 

To reduce the communication payload size, we use stochastic quantization in which we use the quantized version of the information $\boldsymbol{\hat{Q}}_{m}, ~ \forall m \in {\cal N}_n$ to update the primal and dual variables at each worker $n$. 
We also reduces the communication overhead by using the ``censoring idea''. 

We follow a similar stochastic quantization scheme to the one described in \cite{anis2019} where each worker quantizes the difference between its current model and its previously quantized model before transmission ({$\boldsymbol{\theta}_n^k-\boldsymbol{\hat{Q}}_n^{k-1}$\nm}) as {$\boldsymbol{\theta}_n^k-\boldsymbol{\hat{Q}}_n^{k-1} = Q_n(\boldsymbol{\theta}_n^k, \boldsymbol{\hat{Q}}_n^{k-1})$\nm. The function {$Q_n(\cdot)$\nm} is a stochastic quantization operator that depends on the quantization probability $p_{n,i}^k$ for each model vector's dimension $i\in\{1,2,\cdots,d\}$, and on $b_n^k$ bits used for representing each model vector dimension. 

The $i^{\text{th}}$ dimensional element {$[\boldsymbol{\hat{Q}}_n^{k-1}]_i$\nm} of the previously quantized model vector is centred at the quantization range {$2 R_n^k$\nm} that is equally divided into $2^{b_n^k}-1$\nm~quantization levels, yielding the quantization step size $\Delta_n^k=2 R_n^k/(2^{b_n^k}-1)$\nm. In this coordinate, the difference between the $i^{\text{th}}$ dimensional element {$[\boldsymbol{\theta}_n^k]_i$\nm} of the current model vector and {$[\boldsymbol{\hat{Q}}_n^{k-1}]_i$\nm} is
$[c_n(\boldsymbol\theta_n^k)]_i\!=\! \frac{1}{\Delta_n^k} \left([\boldsymbol\theta_n^k]_i-[\boldsymbol{\hat{Q}}_n^{k-1}]_i\!+\!R_n^k\right)\!$ where $R_n^k$\nm~ensures the non-negativity of the quantized value. Then, {$[c_n(\boldsymbol\theta_n^k)]_i$\nm} is mapped to
\begin{align}
[q_n(\boldsymbol\theta_n^k)]_i =\begin{cases}
\left\lceil [c_n(\boldsymbol\theta_n^k)]_i\!\right\rceil & \text{with probability $p_{n,i}^k$}\\[2pt]
\left\lfloor [c_n(\boldsymbol\theta_n^k)]_i\!\right\rfloor & \text{with probability $1-p_{n,i}^k$}, \label{Eq:quant}
\end{cases}
\end{align}\nm
where $\lceil\cdot \rceil$ and $\lfloor\cdot \rfloor$ are the ceiling and floor functions, respectively. Next, the probability $p_{n,i}^k$ in \eqref{Eq:quant} is selected such that the expected quantization error $\mathbb{E}\left[\boldsymbol{e}_{n,i}^k\right]$ is zero  
\vspace{-5pt}\begin{align}
p_{n,i}^k = \left( [c_n(\boldsymbol\theta_n^k)]_i-\lfloor[c_n(\boldsymbol\theta_n^k)]_i \rfloor \right). \label{Eq:Optp}
\end{align}
\noindent The choice of $p_{n,i}^k$ in \eqref{Eq:Optp} ensures that the quantization in \eqref{Eq:quant} is unbiased and the  quantization error variance $\mathbb{E}\left[\left(\bm{e}_{n,i}^k\right)^2\right]$ is less than $({\Delta_n^k})^2$. This implies that $\mathbb{E}\left[\norm{\bm{e}_{n}^k}^2\right] \leq d({\Delta_n^k})^2$. 

In addition to the above condition, the convergence of CQ-GGADMM requires non-increasing quantization step sizes over iterations, \ie $\Delta_n^k \leq \omega \Delta_n^{k-1}$ for all $k$ where $\omega \in (0,1)$. To satisfy this condition, the parameter $b_n^k$ is chosen as
\begin{align}
b_n^k \geq \left\lceil \log_2\left(1 + (2^{b_n^{k-1}}-1)R_n^k/(\omega R_n^{k-1}) \right) \right\rceil. \label{Eq:Optb}
\end{align}\nm
Under this condition, we get that $\Delta_n^k \leq \omega^k \Delta_n^0$. Given $p_{n,i}^k$ in \eqref{Eq:Optp} and $b_n^k$ in \eqref{Eq:Optb}, the convergence of CQ-GGADMM is provided in Section \ref{convergence}. With the aforementioned stochastic quantization procedure, $b_n^k$\nm, $R_n^k$\nm, and $q_n(\boldsymbol{\theta}_n^k)$ suffice to represent $\boldsymbol{\hat{Q}}_n^k$, where $q_n(\boldsymbol{\theta}_n^k)=( [q_n(\boldsymbol{\theta}_n^k)]_1,\dots,[q_n(\boldsymbol{\theta}_n^k)]_d )^\intercal$ which are transmitted to neighbors. After receiving these values, $\boldsymbol{\hat{Q}}_n^k$\nm~can be reconstructed as
$\boldsymbol{\hat{Q}}_n^k = \boldsymbol{\hat{Q}}_n^{k-1}+ \Delta_n^k q_n(\boldsymbol\theta_n^k)-R_n^k\mathbf{1}$. When the full arithmetic precision uses $32$ bits, every transmission payload size of CQ-GGADMM is $b_n^k d + (b_R + b_b)$ bits, where $b_R\leq 32$ and $b_b\leq 32$ are the required bits to represent $R_n^k$ and $b_n^k$, respectively. Compared to GGADMM, whose payload size is $32d$ bits, CQ-GGADMM can achieve a huge reduction in communication overhead, particularly for large models, \ie large $d$. 

Now, we introduce a censoring condition to reduce the number of workers communicating at a given iteration by allowing the worker to transmit only when the difference between the current and previously transmitted value is sufficiently different. However, we apply the censoring not on the model itself but on its quantized value, \ie if the worker is not censored, it transmits its quantized model to its neighbors. According to the communication-censoring strategy, we have that $\bm{\hat{\theta}}_n^{k+1}
= \bm{\hat{Q}}_n^{k+1}$ provided that $\|\bm{\hat{\theta}}_n^{k}- \bm{\hat{Q}}_n^{k+1}\|
\geq \tau_0 \xi^{k+1}$ and $\bm{\hat{\theta}}_n^{k+1}
= \bm{\hat{\theta}}_n^{k}$, otherwise. The CQ-GGADMM algorithm can be written in this case as
\begin{itemize}
\item[$(1)$] Primal variables for head workers are found using
\begin{align} \label{headupdate}
\nonumber \bm{\theta}_{n}^{k+1} &= \underset{\bm{\theta}_{n}}{\arg \min}~ f_n(\bm{\theta}_{n}) + \langle \bm{\theta}_{n} , \bm{\alpha}_{n}^k- \rho \sum_{m \in \mathcal{N}_n} \bm{\hat{\theta}}_{m}^{k}\rangle\\
& + \frac{\rho}{2} d_{n} \|\bm{\theta}_{n}\|^2.
\end{align}
\item[$(2)$] Primal variables update for tail workers is done as follow
\begin{align} \label{tailupdate}
\nonumber \bm{\theta}_{m}^{k+1} &= \underset{\bm{\theta}_{m}}{\arg \min}~ f_m(\bm{\theta}_{m}) + \langle \bm{\theta}_{m} , \bm{\alpha}_{m}^k- \rho \sum_{n \in \mathcal{N}_m} \bm{\hat{\theta}}_{n}^{k+1}\rangle \\
&+ \frac{\rho}{2} d_{m} \|\bm{\theta}_{m}\|^2.
\end{align}
\item[$(3)$] Dual variable of each worker is updated locally
\begin{align} \label{alphaupdate}
\bm{\alpha}_n^{k+1} = \bm{\alpha}_n^{k} + \rho \sum_{m \in \mathcal{N}_n} (\bm{\hat{\theta}}_{n}^{k+1} - \bm{\hat{\theta}}_{m}^{k+1}), ~\forall n \in \mathcal{V}.
\end{align}
\end{itemize}
\begin{algorithm}[t]
{ 				
\begin{algorithmic}[1]
\STATE {\bf Input}: $N, \rho, \tau_0, \xi, f_n(\boldsymbol{\theta}_n) ~ \text{for all} \ n$ 
\STATE $\boldsymbol{\theta}_n^{0}=0, \bm{\hat{\theta}}_{n}^{0}=0, \bm{\alpha}_{n}^{0}=0$ for all $n$
\FOR {$k=0,1,2,\cdots,K$}
\STATE \textbf{Head worker $n \in \mathcal{H}$:} 
\SCOPE
\STATE \textbf{computes} its primal variable $\boldsymbol{\theta}_n^{k+1}$ via \eqref{headupdate} in parallel
\STATE \textbf{quantizes} its primal variable $\boldsymbol{\theta}_n^{k+1}$ to  $\bm{\hat{Q}}_n^{k+1}$ as described in section \ref{CQGADMM}
\IF{$\|\bm{\hat{\theta}}_{n}^{k} - \bm{\hat{Q}}_n^{k+1}\| \geq \tau_0 \xi^{k+1}$}
\STATE worker $n$ \textbf{sends} $q_n(\boldsymbol{\theta}_n^{k+1})$, $R_n^{k+1}$, and $b_n^{k+1}$ to its neighboring workers  $\mathcal{N}_n$ and \textbf{sets} $\bm{\hat{\theta}}_{n}^{k+1} = \bm{\hat{Q}}_n^{k+1}$. 
\ELSE
\STATE worker $n$ \textbf{does not transmit} and \textbf{sets} $\bm{\hat{\theta}}_{n}^{k+1} = \boldsymbol{\hat{\theta}}_n^{k}$. 
\ENDIF
\ENDSCOPE
\STATE \textbf{Tail worker $m \in \mathcal{T}$:} 
\SCOPE
\STATE \textbf{computes} its primal variable  $\boldsymbol{\theta}_m^{k+1}$ via \eqref{tailupdate} in parallel
\STATE \textbf{quantizes} its primal variable $\boldsymbol{\theta}_n^{k+1}$ to  $\bm{\hat{Q}}_n^{k+1}$ as described in section \ref{CQGADMM}
\IF{$\|\bm{\hat{\theta}}_{m}^{k} -  \bm{\hat{Q}}_m^{k+1}\| \geq \tau_0 \xi^{k+1}$}
\STATE worker $m$ \textbf{sends} $q_n(\boldsymbol{\theta}_m^{k+1})$, $R_m^{k+1}$, and $b_m^{k+1}$ to its neighboring workers  $\mathcal{N}_m$ and \textbf{sets} $\bm{\hat{\theta}}_{m}^{k+1} = \bm{\hat{Q}}_m^{k+1}$.  
\ELSE
\STATE worker $m$ \textbf{does not transmit} and \textbf{sets} $\bm{\hat{\theta}}_{m}^{k+1} = \boldsymbol{\hat{\theta}}_m^{k}$. 
\ENDIF
\ENDSCOPE
\STATE \textbf{Every worker updates}  the dual variables $\bm{\alpha}_{n}^{k+1}$ via \eqref{alphaupdate} locally.
\ENDFOR			
\end{algorithmic}
\caption{Censored Quantized Generalized Group ADMM (CQ-GGADMM) \label{algo2}}
}						
\end{algorithm}  
\section{Convergence Analysis}\label{convergence}
Before stating the main results of the paper, we further make the following assumptions.\\
\textbf{Assumption 2.} There exists an optimal solution set to $\textbf{(P1)}$ which has at least one finite element. \\
\textbf{Assumption 3.} The local cost functions $f_n$ are convex.\\
\textbf{Assumption 4.} The local cost functions $f_n$ are strongly convex with parameter $\mu_n > 0$, \ie
\begin{align}
\|\nabla f_n(\bm{x}) - \nabla f_n(\bm{y}) \| \geq \mu_n \|\bm{x}-\bm{y}\|, \forall \bm{x}, \bm{y} \in \mathbb{R}^d.
\end{align} 
\textbf{Assumption 5.} The local cost functions $f_n$ have $L_n$-Lipschitz continuous gradient ($L_n > 0$)
\begin{align}
\|\nabla f_n(\bm{x}) - \nabla f_n(\bm{y}) \| \leq L_n \|\bm{x}-\bm{y}\|, \forall \bm{x}, \bm{y} \in \mathbb{R}^d.
\end{align} 
Assumptions 1-5 are key assumptions that are often used in the context of distributed optimization \cite{liu2019, Konecny2016, chen2018}. While only assumptions 1-3 are needed to prove the convergence of CQ-GGADMM, assumptions 4 and 5 are further required to show the linear convergence rate. Note that Assumption 2 ensures that the problem \textbf{(P2)} has at least one optimal solution, denoted by $\bm{\theta}^\star$. Under Assumption 4, the function $f$ is strongly convex with parameter $\mu = \underset{1 \leq n \leq N}{\min}~\mu_n$, and from Assumption 5, we can see that $f$ has $L$-Lipschitz continuous gradient with $L = \underset{1 \leq n \leq N}{\max}~L_n$.\\
To proceed with the analysis, we start by writing the optimality conditions as
\begin{align}\label{optimality}
\bm{\theta}_{n}^\star = \bm{\theta}_{m}^\star, ~\forall (n, m) \in \mathcal{E} ~~ \text{and}~~ \nabla f_n(\bm{\theta}_n^\star) + \bm{\alpha}_{n}^\star = \bm{0}, ~ \forall n \in \mathcal{V},
\end{align}
where $\bm{\theta}_{n}^\star$ and $\bm{\alpha}_{n}^\star$ are the optimal values of the primal and dual variables, respectively. We define the primal residual $\bm{r}_{n,m}^{k+1} = \bm{\theta}_{n}^{k+1}-\bm{\theta}_{m}^{k+1}, ~\forall (n,m) \in \mathcal{E}$, and the dual residual $\bm{s}_{n}^{k+1} = \rho \sum_{m \in \mathcal{N}_n} (\bm{\hat{\theta}}_{m}^{k+1}-\bm{\hat{\theta}}_{m}^{k}), ~\forall n \in \mathcal{H}$. The total error is defined as $\bm{\epsilon}_{n}^{k+1} = \bm{\theta}_{n}^{k+1}-\bm{\hat{\theta}}_{n}^{k+1}, ~\forall n=1,\dots,N$. The total error can be decomposed as the sum of two errors: $(i)$ a random error coming from the quantization process $\bm{e}_{n}^{k+1} = \bm{\theta}_{n}^{k+1}-\bm{\hat{Q}}_{n}^{k+1}$, and $(ii)$ a deterministic one due to the censoring strategy $\bm{\ell}_{n}^{k+1} = \bm{\hat{Q}}_{n}^{k+1}-\bm{\hat{\theta}}_{n}^{k+1}$. According to the communication-censoring strategy, we have that $\bm{\hat{\theta}}_n^{k} = \bm{\hat{Q}}_n^{k}$ if $\|\bm{\hat{\theta}}_n^{k-1}- \bm{\hat{Q}}_n^{k}\| \geq \tau^{k}$ and $\bm{\hat{\theta}}_n^{k}
= \bm{\hat{\theta}}_n^{k-1}$ if $\|\bm{\hat{\theta}}_n^{k-1}- \bm{\hat{Q}}_n^{k}\| < \tau^{k}$. In both cases, we have
$ \|\bm{\ell}_n^{k} \| = \|\bm{\hat{Q}}_n^{k} - \bm{\hat{\theta}}_n^{k}\| < \tau^k$. Since the sequence $\{\tau^k\}$ is a decreasing non-negative sequence, then we have that $\|\bm{\ell}_n^k \| \leq \tau^k$ and $\|\bm{\ell}_n^{k+1} \| \leq \tau^k$, $\forall n \in \mathcal{V}$. Since the second moment of the quantization error is bounded by
$\mathbb{E}\left[\|\bm{e}_n^k \|^2 \right] \leq d (\Delta_n^k)^2 \leq d (\Delta^0)^2 \omega^{2k}$ where $\Delta^0 = \underset{1 \leq n \leq N}{\max} \Delta_n^0$, then, the total error can be upper bounded, using \eqref{id1}, by
\begin{align}\label{err_bound}
\nonumber \mathbb{E}\left[\|\bm{\epsilon}_n^k \|^2 \right] &\leq 2 (\|\bm{\ell}_n^k \|^2 +  \mathbb{E}\left[\|\bm{e}_n^k \|^2 \right] )\\
&\leq 2 \left(\tau_0^2 \xi^{2k}  + d (\Delta^0)^2 \omega^{2k} \right) \leq 4 C_0^2 \psi^{2k},
\end{align}
where $C_0 = \max\{\tau_0, \sqrt{d} (\Delta^0)\}$, and $\psi = \max\{\xi, \omega \} \in (0, 1)$. 

To prove the convergence of the proposed algorithm, we start by stating and proving the first lemma where we derive upper and lower bounds on the expected value of the  optimality gap.   
\begin{lemma}\label{lemma1}
Under assumptions 1-3, we have the following bounds on the expected value of the optimality gap
\begin{align}
\nonumber & (i)~ \sum_{n=1}^N \mathbb{E}\left[f_n(\bm{\theta}^{k+1})-f_n(\bm{\theta}^\star)\right] \\
\nonumber & \leq\!-\!\sum_{(n,m) \in \mathcal{E}} \mathbb{E}\left[\langle \bm{\lambda}_{n,m}^{k+1}, \bm{r}_{n,m}^{k+1}\rangle \right]\!+\!\sum_{n\in \mathcal{H}} \mathbb{E}\left[\langle \bm{s}_{n}^{k+1}, \bm{\theta}_{n}^\star\!-\!\bm{\theta}_{n}^{k+1}\rangle \right]\\
&+\!\rho \sum_{n=1}^N d_{n} \mathbb{E}\left[\langle \bm{\epsilon}_n^{k+1} , \bm{\theta}_n^\star\!-\!\bm{\theta}_n^{k+1} \rangle \right],\label{ubound} \\
&(ii) ~\sum_{n=1}^N \mathbb{E}\left[f_n(\bm{\theta}^{k+1})\!-\!f_n(\bm{\theta}^\star)\right]\!\geq\!-\sum_{(n,m) \in \mathcal{E}}\!\mathbb{E}\left[\langle \bm{\lambda}_{n,m}^\star, \bm{r}_{n,m}^{k+1}\rangle \right]. \label{lbound}
\end{align}
\end{lemma}
\begin{proof}
The details of the proof are deferred to Appendix \ref{prooflemma1}.
\end{proof}
Next, we present the first theorem that states the asymptotic convergence of the proposed algorithm where we prove the convergence to zero in the mean square sense of both the primal and dual residuals as well as the convergence to zero in the mean sense of the optimality gap.
\begin{theorem}\label{thm1}
Suppose assumptions 1-3 hold, then the CQ-GGADMM iterates lead to 
\begin{enumerate}
\item[$(i)$] the convergence of the primal residual to zero in the mean square sense as $k \rightarrow \infty$, \ie
\begin{align}
\underset{k \rightarrow \infty}{\lim} \mathbb{E}\left[\|\bm{r}_{n,m}^{k}\|^2\right] = 0, ~\forall (n,m) \in \mathcal{E},
\end{align}
\item[$(ii)$] the convergence of the dual residual to zero in the mean square sense as $k \rightarrow \infty$, \ie
\begin{align}
\underset{k \rightarrow \infty}{\lim} \mathbb{E}\left[\|\bm{s}_{n}^{k}\|^2\right] = 0, ~\forall n \in \mathcal{H},
\end{align}
\item[$(iii)$] the convergence of the optimality gap to zero in the mean sense as $k \rightarrow \infty$, \ie
\begin{align}
\underset{k \rightarrow \infty}{\lim} \sum_{n=1}^N \mathbb{E}\left[f_n(\bm{\theta}_{n}^{k}) - f_n(\bm{\theta}_{n}^\star)\right] = 0.
\end{align}
\end{enumerate}
\end{theorem}
\begin{proof}
The proof can be found in Appendix \ref{proofthm1}.
\end{proof}
The linear convergence of the CQ-GGADMM algorithm is presented next.
\begin{theorem}\label{thm2}
Suppose that assumptions 1, 2, 4 and 5 hold and the dual variable $\bm{\alpha}$ is initialized such that $\bm{\alpha}^0$ lies in the column space of the signed incidence matrix $\bm{M}_{-}$. For sufficiently small $\kappa$ and $\rho$, the sequence of iterates of CQ-GGADMM converges linearly with a rate $(1+\delta_2)/2$ where $\delta_2 = \max\{(1+\kappa)^{-1}, \psi^2\}$.
\end{theorem}
\begin{proof}
The proof is provided in Appendix \ref{proofthm2} where the conditions on $\kappa$ and $\rho$ are derived. In the proof, we require an extra initialization condition that $\bm{\alpha}^0$ lies in the column space of $\bm{M}_{-}$, by taking $\bm{\alpha}^0 = \bm{0}$. Thus, we ensure that $\bm{\alpha}^{k}$ will always stay in the column space of $\bm{M}_{-}$ and we can write $\bm{\alpha}^{k} =\bm{M}_{-} \bm{\beta}^k$. The convergence rate, derived in the proof, depends on the network topology through the values of $\sigma_{\max}(\bm{C})$, $\sigma_{\max}(\bm{M}_{-})$ and $\tilde{\sigma}_{\min}(\bm{M}_{-})$, the properties of the local objective functions ($\mu$ and $L$), the penalty parameter $\rho$ but also on the threshold parameter $\xi$ as well as the parameter $\omega$ used to construct the quantization step sizes.
\end{proof}
\section{Numerical Results}
\label{eval}
To validate our theoretical results, we numerically evaluate the performance of CQ-GGADMM compared with GGADMM, C-GGADMM, and C-ADMM}~\cite{liu2019}. Note that C-ADMM performs censoring on top of the Jacobian and decentralized version of the standard ADMM. For the tuning parameters, we choose the values leading to the best performance of all algorithms. 

\noindent \textbf{Model and Datasets.} All simulations are conducted using synthetic and real datasets. For the synthetic data, we used the datasets that were generated in~\cite{chen2018}. We consider two decentralized consensus optimization problems: $(i)$ linear regression, and $(ii)$ logistic regression. The details about the datasets used in our experiments are summarized in Table \ref{table}. For each dataset, the number of samples are uniformly distributed across the $N$ workers.


\begin{table*}[h]
\centering
\begin{tabular}{|l|l|l|c|c|}
\hline
\textbf{Dataset} & \textbf{Task} & \textbf{Data Type} & \textbf{Model Size ($d$)} & \textbf{Number of Instances} \\ \hline \hline
synth-linear \cite{chen2018} & linear regression     & synthetic &50  & 1200 \\ \hline
Body Fat \cite{Dua:2019} & linear regression     & real      & $14$ & $252$ \\ \hline
synth-logistic \cite{chen2018}& logistic regression & synthetic &50  &  1200\\ \hline
Derm \cite{Dua:2019}& logistic regression & real      &  34&  358\\ \hline
\end{tabular}
\caption{List of datasets used in the numerical experiments.}
\label{table} 
\end{table*}


\noindent \textbf{Communication Energy.} We assume that the total system bandwidth $2$MHz is equally divided across workers. Therefore, the available bandwidth to the $n$-th worker ($B_n$) at every communication round when utilizing GGADMM is $(4/N)$MHz since only half of the workers are transmitting at each communication round. On the other hand, the available bandwidth to each worker when using C-ADMM  is $(2/N)$MHz. The power spectral density ($N_0$) is $10^{-6}$W$/$Hz, and each upload/download transmission time ($\tau$) is $1ms$. We assume a free space model, and each worker needs to transmit at a power level that allows transmitting the model vector in one communication round (the rate is bottlenecked by the worst link). For example, using C-ADMM, each worker needs to find the transmission power that achieves the transmission rate $R=(32d/1ms)$ bits/sec. Using Shannon capacity, the corresponding transmission power can be calculated as $P=\tau D^2 N_0 B_n \left(2^{R/B_n}-1\right)$, and the consumed energy will be $E=P \tau$.
\subsection{Linear Regression}
In this case, the local cost function at worker $n$ is explicitly given by $f_n(\bm{\theta}) = \frac{1}{2} \| \bm{X}_n \bm{\theta} - \bm{y}_n\|^2$ where $\bm{X}_n \in \mathbb{R}^{s \times d}$ and $\bm{y}_n \in \mathbb{R}^{s \times 1}$ are private for each worker $n \in \mathcal{V}$ where $s$ represents the size of the data at each worker. Figs.~\ref{Fig_LR_synth}-(a) and \ref{Fig_LR_real}-(a) corroborate that both C-GGADMM and CQ-GGADMM achieve the same convergence speed as GGADMM and significantly outperform C-ADMM, thanks to the the alternation update, censoring, and stochastic quantization. Note that though, C-ADMM allows workers to update their models in parallel, it requires significantly higher number of iterations. Figs.~\ref{Fig_LR_synth}-(b) and \ref{Fig_LR_real}-(b) show that C-GGADMM achieves $10^{-4}$ objective error with the minimum number of communication rounds outperforming all other algorithms. We also note that introducing quantization on top of censoring has increased the number of communication rounds. However, in terms of the total number of transmitted bits and consumed energy, CQ-GGADMM outperforms all algorithms.
\begin{figure*}[t]
\centering
\includegraphics[width=\textwidth]{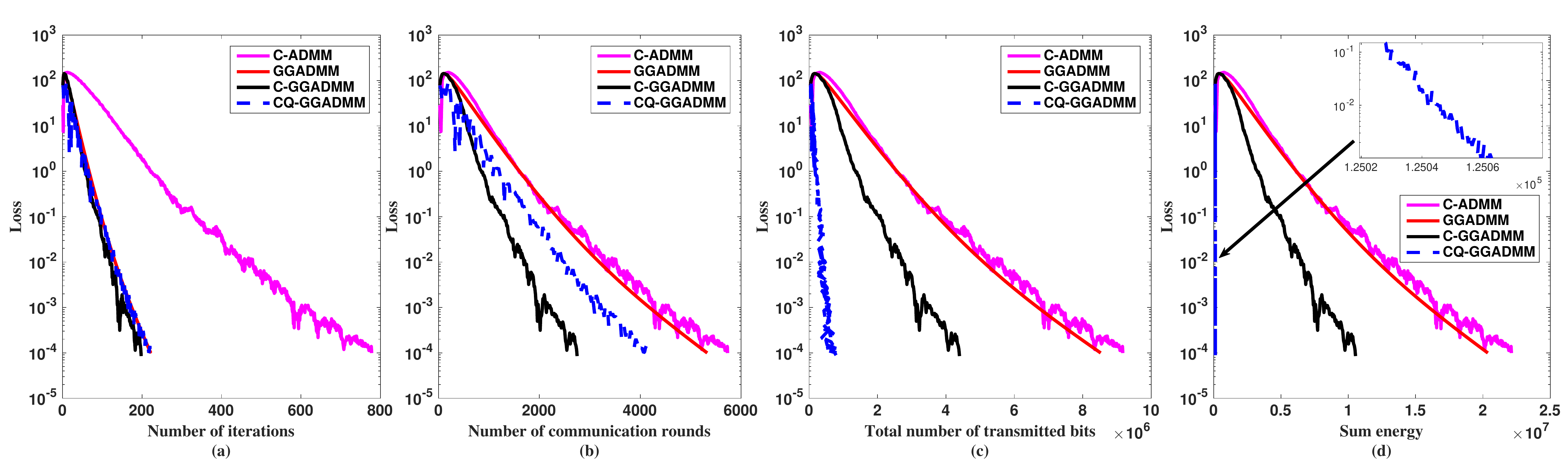}
\caption{\emph{Linear regression} results on synthetic dataset showing loss w.r.t.: (a) \# iterations; (b) \# communication rounds; (c) \# transmitted bits; (d) total energy.}
\label{Fig_LR_synth} 
\end{figure*}
\begin{figure*}[t]
\centering
\includegraphics[width=\textwidth]{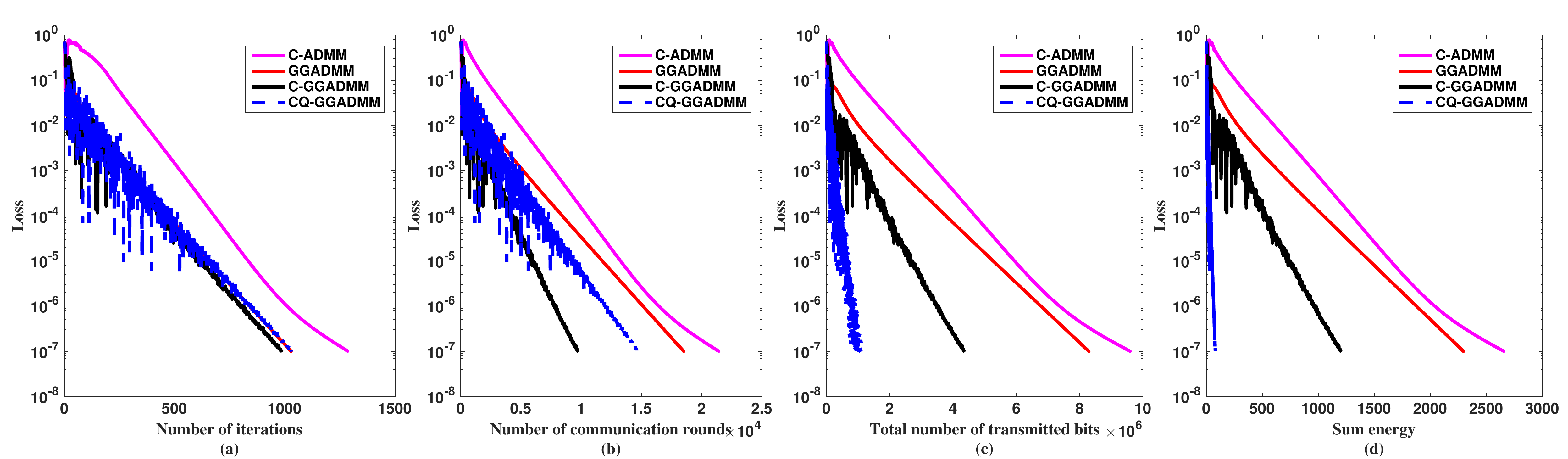}
\caption{\emph{Linear regression} results on real dataset showing loss w.r.t.: (a) \# iterations; (b) \# communication rounds; (c) \# transmitted bits; (d) total energy.}
\label{Fig_LR_real} 
\end{figure*}
\subsection{Logistic Regression}
In this section, we consider the binary logistic regression problem. We assume that worker $n$ owns a data matrix $\bm{X}_n = (\bm{x}_{n,1}, \dots, \bm{x}_{n,s})^T \in \mathbb{R}^{s \times d}$ along with the corresponding labels $\bm{y}_n = (y_{n,1}, \dots, y_{n,s}) \in \{-1, 1\}^{s}$. The local cost function for worker $n$ is then given by $f_n(\bm{\theta}) = \frac{1}{s} \sum_{j=1}^s \log\left( 1 + \exp\left(- y_{n,j} \bm{x}_{n,j}^T \bm{\theta} \right)\right) + \frac{\mu_0}{2} \|\bm{\theta}\|^2$ where $\mu_0$ is the regularization parameter. As observed from Figs.~\ref{Fig_LogReg_synth}-(a) and \ref{Fig_LogReg_real}-(a), C-GADMM requires more iterations compared to GADMM to achieve the same loss which leads to either no saving in the number of communication rounds  (see Fig.~\ref{Fig_LogReg_synth}-(b)) or a small saving in the number of communication rounds (see Fig.~\ref{Fig_LogReg_real}-(b)). It also appears that the update of each individual worker when not quantizing is important at each iteration and censoring hurts the convergence speed. However, interestingly, when introducing stochastic quantization and performing censoring on top of the quantized models, we overcome this issue, and we show significant savings in the number of communication rounds and the communication overhead per iteration. 
\begin{figure*}[t]
\centering
\includegraphics[width=\textwidth]{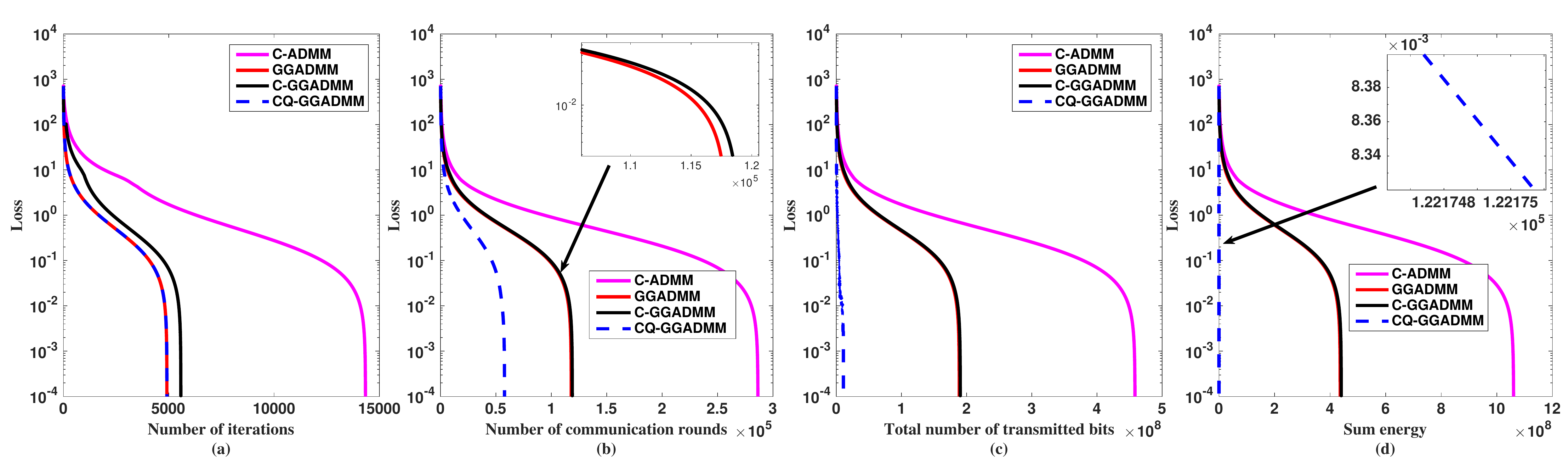}
\caption{\emph{Logistic regression} results on synthetic dataset showing loss w.r.t.: (a) \# iterations; (b) \# communication rounds; (c) \# transmitted bits; (d) total energy.}
\label{Fig_LogReg_synth} 
\end{figure*}
\begin{figure*}[t]
\centering
\includegraphics[width=\textwidth]{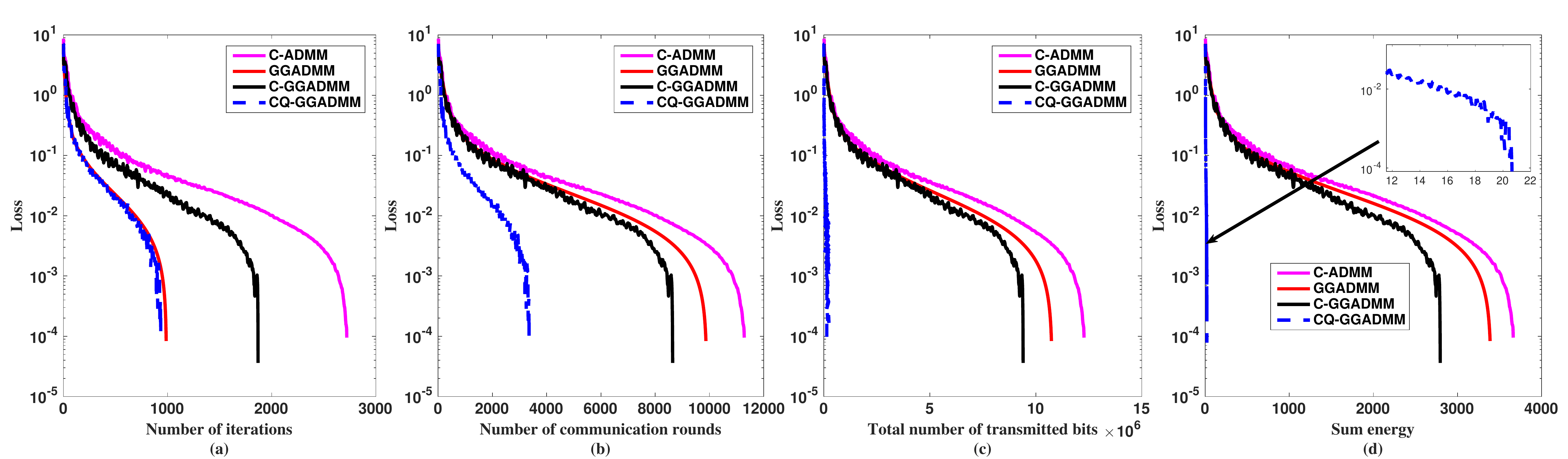}
\caption{\emph{Logistic regression} results on real dataset showing loss w.r.t.: (a) \# iterations; (b) \# communication rounds; (c) \# transmitted bits; (d) total energy.}
\label{Fig_LogReg_real} 
\end{figure*}

\section{Conclusions}\label{SecConc}
In this paper, we have proposed a communication-efficiently decentralized ML algorithm that extends GADMM and Q-GADMM to arbitrary topologies. Moreover, the proposed algorithm leverages censoring (sparsification) to minimize the number of communication rounds for each worker. Utilizing a decreasing sequence of censoring threshold, stochastic quantization, and adjusting the quantization range at every iteration such that a linear convergence rate is achieved are key features that make CQ-GGADMM robust to errors while ensuring its convergence guarantees. Numerical results in convex linear and logistic regression tasks corroborate the advantages of CQ-GGADMM over GGADMM, and C-ADMM.
\section{Appendices}
\subsection{Basic identities and inequalities}\label{identities_inequalities}
\noindent For any two vectors $\bm{x}$, $\bm{y} \in \mathbb{R}^d$, we have 
\small
\begin{align}
\|\bm{x} + \bm{y} \|^2 &\leq 2 \left( \| \bm{x} \|^2 + \| \bm{y} \|^2 \right), ~\forall \bm{x}, \bm{y} \in \mathbb{R}^d, \label{id1} \\
2 \langle \bm{x}, \bm{y} \rangle &\leq \frac{1}{\eta} \| \bm{x} \|^2 + \eta \| \bm{y} \|^2, ~\forall \bm{x}, \bm{y} \in \mathbb{R}^d, ~\eta > 0, \label{id2}
\end{align}
\normalsize
\small
For any two matrices $\bm{A}$ and $\bm{B}$, we have
\begin{align}
2 \langle \bm{A}, \bm{B} \rangle &\leq \eta \|\bm{A}\|_F^2 + \frac{1}{\eta} \|\bm{B}\|_F^2,  ~\forall \eta > 0, \label{id4}\\
\|\bm{A} \bm{B}\|_F &\leq \sigma_{\max}(\bm{A}) \|\bm{B}\|_F,
\label{id5} \\
\|\bm{A} + \bm{B}\|_F^2 &\leq \eta \|\bm{A}\|_F^2 + \frac{\eta}{\eta-1} \|\bm{B}\|_F^2, ~\forall \eta > 1, \label{id7}
\end{align}
\normalsize
where $\sigma_{\max}(\bm{A})$ denotes the maximum singular value of the matrix $A$.
\subsection{Proof of Lemma \ref{lemma1}}\label{prooflemma1}
Using \eqref{headupdate} the update of the head workers can be written as
\small
\begin{align}\label{head_eq}
\nabla f_n(\bm{\theta}_n^{k+1}) + \bm{\alpha}_n^k - \rho \sum_{m \in \mathcal{N}_n} \bm{\hat{\theta}}_m^{k} + \rho d_{n} \bm{\theta}_n^{k+1} = \bm{0}.
\end{align}
\normalsize
Using the update of $\bm{\alpha}_n^k$ in \eqref{alphaupdate}, and the definition of the dual residual, we get
\small
\begin{align}
\nabla f_n(\bm{\theta}_n^{k+1}) + \bm{\alpha}_n^{k+1} + \rho d_{n} \bm{\epsilon}_n^{k+1} + \bm{s}_n^{k+1}  = \bm{0}.
\end{align}
\normalsize
Thus, $\bm{\theta}_n^{k+1}$ minimizes the function $f_n(\bm{\theta}_n) + \langle \bm{\alpha}_n^{k+1} +  \rho d_{n} \bm{\epsilon}_n^{k+1}  + \bm{s}_n^{k+1}, \bm{\theta}_n \rangle$ and as a consequence
\small
\begin{align}
\nonumber &\mathbb{E}\left[f_n(\bm{\theta}_n^{k+1})\!+\!\langle \bm{\alpha}_n^{k+1}\!+\!\rho d_{n} \bm{\epsilon}_n^{k+1}\!+\! \bm{s}_n^{k+1}, \bm{\theta}_n^{k+1} \rangle \right]\!\\
&\leq\! \mathbb{E}\left[f_n(\bm{\theta}_n^\star)\!+\!\langle \bm{\alpha}_n^{k+1}\!+\!\rho d_{n} \bm{\epsilon}_n^{k+1}\!+\! \bm{s}_n^{k+1}, \bm{\theta}_n^\star \rangle \right].
\end{align}
\normalsize
Similarly, using the update of the tail workers as in  \eqref{tailupdate}, we can write
\small
\begin{align}\label{tail_eq}
\nabla f_m(\bm{\theta}_m^{k+1}) + \bm{\alpha}_m^k - \rho \sum_{n \in \mathcal{N}_m} \bm{\hat{\theta}}_n^{k+1} + \rho d_{m} \bm{\theta}_m^{k+1} = \bm{0}.
\end{align}
\normalsize
Hence, we get $\nabla f_m(\bm{\theta}_m^{k+1}) + \bm{\alpha}_m^{k+1} + \rho d_{m} \bm{\epsilon}_m^{k+1}  = \bm{0}$. Thus, we can observe that the dual feasibility condition is fulfilled by the tail workers and $\bm{\theta}_m^{k+1}$ minimizes $f_m(\bm{\theta}_m) + \langle \bm{\alpha}_m^{k+1} + \rho d_{m} \bm{\epsilon}_m^{k+1}, \bm{\theta}_m \rangle$. Therefore, we obtain the following inequality
\small
\begin{align}
\nonumber &\mathbb{E}\left[f_m(\bm{\theta}_m^{k+1})\right]\!+\!\mathbb{E}\left[\langle \bm{\alpha}_m^{k+1}\!+\!\rho d_{m} \bm{\epsilon}_m^{k+1}, \bm{\theta}_m^{k+1} \rangle \right]\\
&\leq\!\mathbb{E}\left[f_m(\bm{\theta}_m^\star)\right]\!+\!\mathbb{E}\left[\langle \bm{\alpha}_m^{k+1}\!+\!\rho d_{m} \bm{\epsilon}_m^{k+1}, \bm{\theta}_m^\star \rangle \right].
\end{align}
\normalsize
Summing over all workers, we get
\small
\begin{align}\label{peq}
\nonumber &\sum_{n=1}^N \mathbb{E}\left[f_n(\bm{\theta}_n^{k+1})-f_n(\bm{\theta}_n^\star)\right]
\\ \nonumber & \leq \sum_{n \in \mathcal{H}} \mathbb{E}\left[\langle \bm{\alpha}_n^{k+1} + \bm{s}_n^{k+1} + \rho d_{n} \bm{\epsilon}_n^{k+1}, \bm{\theta}_n^\star - \bm{\theta}_n^{k+1} \rangle \right]  \\
\nonumber & + \sum_{m \in \mathcal{T}} \mathbb{E}\left[\langle \bm{\alpha}_m^{k+1} + \rho d_{m} \bm{\epsilon}_m^{k+1}, \bm{\theta}_m^\star - \bm{\theta}_m^{k+1} \rangle \right] \\
&+ \rho \sum_{n=1}^N \mathbb{E}\left[\langle d_{n} \bm{\epsilon}_n^{k+1} , \bm{\theta}_n^\star - \bm{\theta}_n^{k+1} \rangle \right].
\end{align}
\normalsize
Using the definition of $\bm{\alpha}_n^{k+1}, ~n\in \mathcal{V}$ in the right hand-side of \eqref{peq}
\small
\begin{align}
\nonumber &\sum_{n \in \mathcal{H}} \mathbb{E}\left[\langle \bm{\alpha}_n^{k+1}, \bm{\theta}_n^\star\!-\!\bm{\theta}_n^{k+1} \rangle \right]\!+\!\sum_{m \in \mathcal{T}} \mathbb{E}\left[\langle \bm{\alpha}_m^{k+1}, \bm{\theta}_m^\star\!-\!\bm{\theta}_m^{k+1} \rangle \right] \\
&=\!\sum_{(n,m) \in \mathcal{E}} \mathbb{E}\left[\langle \bm{\lambda}_{n,m}^{k+1}, \bm{\theta}_n^\star\!-\!\bm{\theta}_n^{k+1} \rangle \right]\!+\!\sum_{(n,m) \in \mathcal{E}} \mathbb{E}\left[\langle \bm{\lambda}_{m,n}^{k+1}, \bm{\theta}_m^\star\!-\!\bm{\theta}_m^{k+1} \rangle \right].
\end{align}
\normalsize
Using that $\bm{\lambda}_{m,n}^{k+1} = - \bm{\lambda}_{n,m}^{k+1}$, and that $\bm{\theta}_{n}^\star = \bm{\theta}_{m}^\star,~\forall (n, m) \in \mathcal{E}$  we can write
\small
\begin{align}
&\nonumber \sum_{n \in \mathcal{H}} \mathbb{E}\left[\langle \bm{\alpha}_n^{k+1}, \bm{\theta}_n^\star - \bm{\theta}_n^{k+1} \rangle \right] + \sum_{m \in \mathcal{T}} \mathbb{E}\left[\langle \bm{\alpha}_m^{k+1}, \bm{\theta}_m^\star - \bm{\theta}_m^{k+1} \rangle \right] \\
&= - \sum_{(n,m) \in \mathcal{E}} \mathbb{E}\left[\langle \bm{\lambda}_{n,m}^{k+1}, \bm{r}_{n,m}^{k+1} \rangle \right].
\end{align}
\normalsize
This proves $(i)$ of Lemma \ref{lemma1}. To prove $(ii)$, we know from the optimality conditions that $\nabla f_n(\bm{\theta}_n^\star) + \bm{\alpha}_{n}^\star = \bm{0}$. Thus, $\bm{\theta}_n^{\star}$ minimizes the function $f_n(\bm{\theta}_n) + \langle \bm{\alpha}_n^{\star}, \bm{\theta}_n \rangle$ and for $n \in \mathcal{H}$
\small
\begin{align}
\mathbb{E}\left[f_n(\bm{\theta}_n^{\star})\right] + \mathbb{E}\left[\langle \bm{\alpha}_n^{\star}, \bm{\theta}_n^{\star} \rangle \right] \leq \mathbb{E}\left[f_n(\bm{\theta}_n^{k+1})\right] + \mathbb{E}\left[\langle \bm{\alpha}_n^{\star}, \bm{\theta}_n^{k+1} \rangle \right].
\end{align}
\normalsize
Similarly, we have, for $m \in \mathcal{T}$, that 
\small
\begin{align}
\mathbb{E}\left[f_m(\bm{\theta}_m^{\star})\right]\!+\! \mathbb{E}\left[\langle \bm{\alpha}_m^{\star}, \bm{\theta}_m^{\star} \rangle \right]\!\leq\!\mathbb{E}\left[f_m(\bm{\theta}_m^{k+1})\right]\!+\! \mathbb{E}\left[\langle \bm{\alpha}_m^{\star}, \bm{\theta}_m^{k+1} \rangle \right].
\end{align}
\normalsize
Summing over all workers, we get
\small
\begin{align}
\nonumber &\sum_{n=1}^N \mathbb{E}\left[f_n(\bm{\theta}_n^{k+1})-f_n(\bm{\theta}_n^\star)\right] \\
& \nonumber \geq \sum_{n \in \mathcal{H}} \mathbb{E}\left[\langle \bm{\alpha}_n^{k+1}, \bm{\theta}_n^\star - \bm{\theta}_n^{k+1} \rangle \right] + \sum_{m \in \mathcal{T}} \mathbb{E}\left[\langle \bm{\alpha}_m^{k+1}, \bm{\theta}_m^\star - \bm{\theta}_m^{k+1} \rangle \right] 
\\\nonumber &\overset{\mathrm{(a)}}{\geq}\!\sum_{(n,m) \in \mathcal{E}} \!\mathbb{E}\left[\langle \bm{\lambda}_{n,m}^{\star}, \bm{\theta}_n^\star\!-\!\bm{\theta}_n^{k+1} \rangle \right]\!+\!\sum_{(n,m) \in \mathcal{E}}\mathbb{E}\left[ \langle \bm{\lambda}_{m,n}^{\star}, \bm{\theta}_m^\star\!-\! \bm{\theta}_m^{k+1} \rangle \right] \\
& \overset{\mathrm{(b)}}{\geq}\!-\!\sum_{(n,m) \in \mathcal{E}}\!\mathbb{E}\left[\langle \bm{\lambda}_{n,m}^{\star}, \bm{r}_{n,m}^{k+1} \rangle \right],
\end{align}
\normalsize
where we used the definition of $\bm{\alpha}_n^\star$ in $\mathrm{(a)}$ and that $\bm{\lambda}_{m,n}^{k+1} = - \bm{\lambda}_{n,m}^{k+1}$, and that $\bm{\theta}_{n}^\star = \bm{\theta}_{m}^\star$ in $\mathrm{(b)}$.
\subsection{Proof of Theorem \ref{thm1}}\label{proofthm1}
Multiplying \eqref{lbound} by (-1), adding \eqref{ubound} and multiplying the sum by 2, we get
\small
\begin{align}\label{c}
\nonumber &\!2\!\sum_{(n,m) \in \mathcal{E}}\!\mathbb{E}\left[\langle \bm{\lambda}_{n,m}^{\star}\!-\!\bm{\lambda}_{n,m}^{k+1}, \bm{r}_{n,m}^{k+1}\rangle \right]\!+\!2 \sum_{n\in \mathcal{H}} \mathbb{E}\left[\langle \bm{s}_{n}^{k+1}, \bm{\theta}_{n}^\star-\bm{\theta}_{n}^{k+1}\rangle \right]\!\\
&+\!2 \rho \sum_{n=1}^N \mathbb{E}\left[\langle d_{n} \bm{\epsilon}_n^{k+1} , \bm{\theta}_n^\star - \bm{\theta}_n^{k+1} \rangle \right]\!\geq\!0.
\end{align}
\normalsize
Since $\bm{\lambda}_{n,m}^{k+1} = \bm{\lambda}_{n,m}^{k} + \rho \bm{r}_{n,m}^{k+1} + \rho (\bm{\epsilon}_m^{k+1}- \bm{\epsilon}_n^{k+1})$, then we can write
\small
\begin{align} \label{lambdaeq}
\nonumber &2 \sum_{(n,m) \in \mathcal{E}} \mathbb{E}\left[\langle \bm{\lambda}_{n,m}^{\star}-\bm{\lambda}_{n,m}^{k+1}, \bm{r}_{n,m}^{k+1}\rangle \right] \\
\nonumber &=\!2\!\sum_{(n,m) \in \mathcal{E}}\!\mathbb{E}\left[\langle \bm{\lambda}_{n,m}^{\star}-\bm{\lambda}_{n,m}^{k}, \bm{r}_{n,m}^{k+1}\rangle \right]\!-\!2\rho\!\sum_{(n,m) \in \mathcal{E}}\!\mathbb{E}\left[\|\bm{r}_{n,m}^{k+1}\|^2\right]\!\\
&-\!2 \rho\!\sum_{(n,m) \in \mathcal{E}}\!\mathbb{E}\left[\langle \bm{\epsilon}_m^{k+1}\!-\!\bm{\epsilon}_n^{k+1}, \bm{r}_{n,m}^{k+1}\rangle \right].
\end{align}
\normalsize
Using $\bm{r}_{n,m}^{k+1} = \frac{1}{\rho}(\bm{\lambda}_{n,m}^{k+1}-\bm{\lambda}_{n,m}^{\star}) - \frac{1}{\rho}(\bm{\lambda}_{n,m}^{k}-\bm{\lambda}_{n,m}^{\star}) + \bm{\epsilon}_n^{k+1}- \bm{\epsilon}_m^{k+1}$,
we will examine the different terms of \eqref{lambdaeq} starting from the first term
\small
\begin{align}\label{a1}
\nonumber & 2 \sum_{(n,m) \in \mathcal{E}} \mathbb{E}\left[\langle \bm{\lambda}_{n,m}^{\star}-\bm{\lambda}_{n,m}^{k}, \bm{r}_{n,m}^{k+1}\rangle \right] \\
\nonumber &= \frac{2}{\rho} \sum_{(n,m) \in \mathcal{E}} \mathbb{E}\left[\langle \bm{\lambda}_{n,m}^{\star}-\bm{\lambda}_{n,m}^{k}, \bm{\lambda}_{n,m}^{k+1}-\bm{\lambda}_{n,m}^{\star}\rangle \right] \\
\nonumber &+ \frac{2}{\rho} \sum_{(n,m) \in \mathcal{E}} \mathbb{E}\left[\|\bm{\lambda}_{n,m}^{k}-\bm{\lambda}_{n,m}^{\star}\|^2 \right] \\
& + 2 \sum_{(n,m) \in \mathcal{E}} \mathbb{E}\left[\langle \bm{\lambda}_{n,m}^{\star}-\bm{\lambda}_{n,m}^{k}, \bm{\epsilon}_{n}^{k+1}-\bm{\epsilon}_{m}^{k+1}\rangle \right].
\end{align} 
\normalsize
The second term can be re-written as
\small
\begin{align}\label{a2}
\nonumber & - 2 \rho \sum_{(n,m) \in \mathcal{E}} \mathbb{E}\left[\|\bm{r}_{n,m}^{k+1}\|^2 \right]\\
\nonumber &= -\rho \sum_{(n,m) \in \mathcal{E}} \mathbb{E}\left[\|\bm{r}_{n,m}^{k+1}\|^2\right] - \frac{1}{\rho} \sum_{(n,m) \in \mathcal{E}} \mathbb{E}\left[\|\bm{\lambda}_{n,m}^{k+1}-\bm{\lambda}_{n,m}^{\star}\|^2 \right]\\
\nonumber & - \frac{1}{\rho} \sum_{(n,m) \in \mathcal{E}} \mathbb{E}\left[\|\bm{\lambda}_{n,m}^{k}-\bm{\lambda}_{n,m}^{\star}\|^2\right] - \rho \sum_{(n,m) \in \mathcal{E}} \mathbb{E}\left[\|\bm{\epsilon}_{n}^{k+1}-\bm{\epsilon}_{m}^{k+1}\|^2 \right]  \\
\nonumber & + \frac{2}{\rho} \sum_{(n,m) \in \mathcal{E}} \mathbb{E}\left[\langle \bm{\lambda}_{n,m}^{k+1}-\bm{\lambda}_{n,m}^{\star}, \bm{\lambda}_{n,m}^{k}-\bm{\lambda}_{n,m}^{\star}\rangle \right]  \\
\nonumber & - 2 \sum_{(n,m) \in \mathcal{E}} \mathbb{E}\left[\langle \bm{\lambda}_{n,m}^{k+1}-\bm{\lambda}_{n,m}^{\star}, \bm{\epsilon}_{n}^{k+1}-\bm{\epsilon}_{m}^{k+1}
\rangle \right]\\
& + 2 \sum_{(n,m) \in \mathcal{E}} \mathbb{E}\left[\langle \bm{\lambda}_{n,m}^{k}-\bm{\lambda}_{n,m}^{\star}, \bm{\epsilon}_{n}^{k+1}-\bm{\epsilon}_{m}^{k+1} \rangle \right].
\end{align}
\normalsize
The third term can be expanded as
\small
\begin{align}\label{a3}
\nonumber & - 2 \rho \sum_{(n,m) \in \mathcal{E}} \mathbb{E}\left[\langle \bm{\epsilon}_m^{k+1}- \bm{\epsilon}_n^{k+1}, \bm{r}_{n,m}^{k+1}\rangle \right] \\
\nonumber & = - 2 \sum_{(n,m) \in \mathcal{E}} \mathbb{E}\left[\langle \bm{\epsilon}_m^{k+1}- \bm{\epsilon}_n^{k+1}, \bm{\lambda}_{n,m}^{k+1}-\bm{\lambda}_{n,m}^{\star}\rangle \right]\\
\nonumber & + 2 \sum_{(n,m) \in \mathcal{E}} \mathbb{E}\left[\langle \bm{\epsilon}_m^{k+1}- \bm{\epsilon}_n^{k+1}, \bm{\lambda}_{n,m}^{k}-\bm{\lambda}_{n,m}^{\star}\rangle \right]\\
& + 2 \rho \sum_{(n,m) \in \mathcal{E}} \mathbb{E}\left[\| \bm{\epsilon}_n^{k+1}- \bm{\epsilon}_m^{k+1} \|^2 \right].
\end{align}
\normalsize
From Eqs. \eqref{a1}-\eqref{a3}, we can re-write \eqref{lambdaeq} as
\small
\begin{align}\label{a}
\nonumber &2 \sum_{(n,m) \in \mathcal{E}} \mathbb{E}\left[\langle \bm{\lambda}_{n,m}^{\star}-\bm{\lambda}_{n,m}^{k+1}, \bm{r}_{n,m}^{k+1}\rangle \right] \\
\nonumber &=  \frac{1}{\rho} \sum_{(n,m) \in \mathcal{E}} \!\mathbb{E}\left[\|\bm{\lambda}_{n,m}^{k}\!-\!\bm{\lambda}_{n,m}^{\star}\|^2 \right]\!-\!\frac{1}{\rho} \sum_{(n,m) \in \mathcal{E}}\!\mathbb{E}\left[\|\bm{\lambda}_{n,m}^{k+1}\!-\!\bm{\lambda}_{n,m}^{\star}\|^2\right]\\
\nonumber& -\rho \sum_{(n,m) \in \mathcal{E}}\!\mathbb{E}\left[\|\bm{r}_{n,m}^{k+1}\|^2 \right]\!+\!2 \sum_{(n,m) \in \mathcal{E}}\!\mathbb{E}\left[\langle \bm{\epsilon}_m^{k+1}\!-\!\bm{\epsilon}_n^{k+1}, \bm{\lambda}_{n,m}^{k}\!-\!\bm{\lambda}_{n,m}^{\star}\rangle \right]\\
& + \rho \sum_{(n,m) \in \mathcal{E}} \mathbb{E}\left[\| \bm{\epsilon}_n^{k+1}- \bm{\epsilon}_m^{k+1} \|^2 \right].
\end{align}
\normalsize
The second term of the left hand-side of \eqref{c} can be decomposed as
\small
\begin{align}\label{*}
\nonumber & 2 \sum_{n\in \mathcal{H}} \mathbb{E}\left[\langle \bm{s}_{n}^{k+1}, \bm{\theta}_{n}^\star-\bm{\theta}_{n}^{k+1}\rangle\right] \\
\nonumber&= 2 \rho  \sum_{(n,m) \in \mathcal{E}} \mathbb{E}\left[ \langle \bm{\hat{\theta}}_{m}^{k+1}-\bm{\hat{\theta}}_{m}^{k}, \bm{\theta}_{n}^\star-\bm{\theta}_{m}^{k+1} \rangle \right]\\
& - 2 \rho  \sum_{(n,m) \in \mathcal{E}} \mathbb{E}\left[ \langle \bm{\hat{\theta}}_{m}^{k+1}-\bm{\hat{\theta}}_{m}^{k}, \bm{r}_{n,m}^{k+1} \rangle \right].
\end{align}
\normalsize
Now, we can re-write the first term as
\small
\begin{align}
\nonumber & - 2 \rho  \sum_{(n,m) \in \mathcal{E}}  \mathbb{E}\left[ \langle \bm{\hat{\theta}}_{m}^{k+1}-\bm{\hat{\theta}}_{m}^{k}, \bm{r}_{n,m}^{k+1} \rangle \right]\\
\nonumber &= - 2 \rho  \sum_{(n,m) \in \mathcal{E}}  \mathbb{E}\left[ \langle \bm{\theta}_{m}^{k+1}-\bm{\theta}_{m}^{k}, \bm{r}_{n,m}^{k+1} \rangle \right]\\
& + 2 \rho  \sum_{(n,m) \in \mathcal{E}} \mathbb{E}\left[  \langle \bm{\epsilon}_m^{k+1} - \bm{\epsilon}_m^{k} , \bm{r}_{n,m}^{k+1} \rangle \right].
\end{align}
\normalsize
The second term can be expanded as
\small
\begin{align}
\nonumber & 2 \rho  \sum_{(n,m) \in \mathcal{E}} \mathbb{E}\left[  \langle \bm{\hat{\theta}}_{m}^{k+1}-\bm{\hat{\theta}}_{m}^{k}, \bm{\theta}_{n}^\star-\bm{\theta}_{m}^{k+1} \rangle \right] \\
\nonumber& = 2 \rho  \sum_{(n,m) \in \mathcal{E}}  \mathbb{E}\left[ \langle \bm{\theta}_{m}^{k+1}-\bm{\theta}_{m}^{k}, \bm{\theta}_{n}^\star-\bm{\theta}_{m}^{k+1} \rangle \right]\\
& +  2 \rho  \sum_{(n,m) \in \mathcal{E}}  \mathbb{E}\left[ \langle \bm{\epsilon}_{m}^{k+1}-\bm{\epsilon}_{m}^{k}, \bm{\theta}_{m}^{k+1}-\bm{\theta}_{n}^\star \rangle \right].
\end{align}
\normalsize
Since $\bm{\theta}_n^{*} = \bm{\theta}_m^{*}$, $\forall (n,m) \in \mathcal{E}$ and $\bm{\theta}_{m}^\star-\bm{\theta}_{m}^{k+1} = \bm{\theta}_{m}^\star-\bm{\theta}_{m}^{k} + \bm{\theta}_{m}^{k}-\bm{\theta}_{m}^{k+1}$, we can write
\small
\begin{align}
\nonumber & 2 \rho  \sum_{(n,m) \in \mathcal{E}}  \mathbb{E}\left[\langle \bm{\theta}_{m}^{k+1}-\bm{\theta}_{m}^{k}, \bm{\theta}_{n}^\star-\bm{\theta}_{m}^{k+1} \rangle \right] \\
\nonumber &= -\rho \sum_{(n,m) \in \mathcal{E}}\!\mathbb{E}\left[\|\bm{\theta}_{m}^{k+1}\!-\!\bm{\theta}_{m}^{k}\|^2 \right]\!-\!\rho \sum_{(n,m) \in \mathcal{E}}\!\mathbb{E}\left[ \|\bm{\theta}_{m}^{k+1}\!-\!\bm{\theta}_{m}^{\star}\|^2 \right]\\
&  - \rho \sum_{(n,m) \in \mathcal{E}}\!\mathbb{E}\left[\|\bm{\theta}_{m}^{k}\!-\!\bm{\theta}_{m}^{\star}\|^2 \right]\!+\!2 \rho  \sum_{(n,m) \in \mathcal{E}}\!\mathbb{E}\left[\langle \bm{\theta}_{m}^{k+1}\!-\!\bm{\theta}_{m}^{\star}, \bm{\theta}_{m}^{k}\!-\!\bm{\theta}_{m}^{\star} \rangle \right].
\end{align}
\normalsize
With this expression at hand, we can go back to \eqref{*} 
\small
\begin{align}\label{b}
\nonumber & 2 \sum_{n\in \mathcal{H}} \mathbb{E}\left[\langle \bm{s}_{n}^{k+1}, \bm{\theta}_{n}^\star-\bm{\theta}_{n}^{k+1}\rangle \right] \\
\nonumber &= \rho  \sum_{(n,m) \in \mathcal{E}} \mathbb{E}\left[\|\bm{\theta}_{m}^{k}-\bm{\theta}_{m}^{\star}\|^2 \right] -  \rho  \sum_{(n,m) \in \mathcal{E}} \mathbb{E}\left[\|\bm{\theta}_{m}^{k+1}-\bm{\theta}_{m}^{\star}\|^2 \right]\\
\nonumber & + 2 \rho  \sum_{(n,m) \in \mathcal{E}} \mathbb{E}\left[ \langle \bm{\epsilon}_{m}^{k+1}-\bm{\epsilon}_{m}^{k}, \bm{\theta}_{m}^{k+1}-\bm{\theta}_{n}^\star \rangle \right] \\ 
\nonumber & -\!2 \rho  \sum_{(n,m) \in \mathcal{E}} \mathbb{E}\left[ \langle \bm{\theta}_{m}^{k+1}\!-\!\bm{\theta}_{m}^{k}, \bm{r}_{n,m}^{k+1} \rangle \right]\!+\!2 \rho  \sum_{(n,m) \in \mathcal{E}} \mathbb{E}\left[ \langle \bm{\epsilon}_m^{k+1}\!-\!\bm{\epsilon}_m^{k} , \bm{r}_{n,m}^{k+1} \rangle \right] \\
& \!-\!\rho  \sum_{(n,m) \in \mathcal{E}} \mathbb{E}\left[\|\bm{\theta}_{m}^{k+1}\!-\!\bm{\theta}_{m}^{k}\|^2 \right].
\end{align}
\normalsize
Replacing \eqref{a} and \eqref{b} in \eqref{c}, we obtain
\small
\begin{align}\label{eqq1}
\nonumber & \frac{1}{\rho}\!\sum_{(n,m) \in \mathcal{E}} \!\mathbb{E}\left[\|\bm{\lambda}_{n,m}^{k}\!-\!\bm{\lambda}_{n,m}^{\star}\|^2 \right]\!-\!\frac{1}{\rho} \sum_{(n,m) \in \mathcal{E}}\!\mathbb{E}\left[\|\bm{\lambda}_{n,m}^{k+1}\!-\!\bm{\lambda}_{n,m}^{\star}\|^2\right]\\
\nonumber & \!-\!\rho \sum_{(n,m) \in \mathcal{E}}\!\mathbb{E}\left[\|\bm{r}_{n,m}^{k+1}\|^2 \right]\!+\!2 \sum_{(n,m) \in \mathcal{E}}\!\mathbb{E}\left[\langle \bm{\epsilon}_m^{k+1}\!-\!\bm{\epsilon}_n^{k+1}, \bm{\lambda}_{n,m}^{k}\!-\!\bm{\lambda}_{n,m}^{\star}\rangle \right] \\
\nonumber & \!+\!\rho \sum_{(n,m) \in \mathcal{E}} \mathbb{E}\left[\| \bm{\epsilon}_n^{k+1}\!-\!\bm{\epsilon}_m^{k+1} \|^2 \right]\!+\!\rho  \sum_{(n,m) \in \mathcal{E}} \mathbb{E}\left[\|\bm{\theta}_{m}^{k}\!-\!\bm{\theta}_{m}^{\star}\|^2 \right] \\
\nonumber &\!-\!\rho  \sum_{(n,m) \in \mathcal{E}} \mathbb{E}\left[\|\bm{\theta}_{m}^{k+1}\!-\!\bm{\theta}_{m}^{\star}\|^2 \right]\!-\!\rho  \sum_{(n,m) \in \mathcal{E}} \mathbb{E}\left[\|\bm{\theta}_{m}^{k+1}\!-\!\bm{\theta}_{m}^{k}\|^2 \right] \\
\nonumber & \!-\!2 \rho  \sum_{(n,m) \in \mathcal{E}}  \mathbb{E}\left[\langle \bm{\theta}_{m}^{k+1}\!-\!\bm{\theta}_{m}^{k}, \bm{r}_{n,m}^{k+1} \rangle \right]\!+\!2 \rho  \sum_{(n,m) \in \mathcal{E}}  \mathbb{E}\left[\langle \bm{\epsilon}_m^{k+1}\!-\!\bm{\epsilon}_m^{k} , \bm{r}_{n,m}^{k+1} \rangle \right]  \\
&\!+\!2 \rho  \sum_{(n,m) \in \mathcal{E}} \mathbb{E}\left[ \langle \bm{\epsilon}_{m}^{k+1}\!-\!\bm{\epsilon}_{m}^{k}, \bm{\theta}_{m}^{k+1}\!-\!\bm{\theta}_{n}^\star \rangle \right]\!+\!2 \rho \sum_{n=1}^N \mathbb{E}\left[\langle d_{n} \bm{\epsilon}_n^{k+1} , \bm{\theta}_n^\star\!-\!\bm{\theta}_n^{k+1} \rangle \right]\!\geq\!0.
\end{align}
\normalsize
Using the identity $\bm{r}_{n,m}^{k+1} = \frac{1}{\rho}(\bm{\lambda}_{n,m}^{k+1}-\bm{\lambda}_{n,m}^{k}) + \bm{\epsilon}_n^{k+1}- \bm{\epsilon}_m^{k+1}$,
we can write
\small
\begin{align}
\nonumber& - \rho \sum_{(n,m) \in \mathcal{E}} \mathbb{E}\left[ \|\bm{r}_{n,m}^{k+1}\|^2 \right] \\
\nonumber& = - \frac{1}{\rho} \sum_{(n,m) \in \mathcal{E}}  \mathbb{E}\left[\|\bm{\lambda}_{n,m}^{k+1}-\bm{\lambda}_{n,m}^{k}\|^2 \right] - \rho \sum_{(n,m) \in \mathcal{E}}  \mathbb{E}\left[\|\bm{\epsilon}_{n}^{k+1}-\bm{\epsilon}_{m}^{k+1}\|^2 \right]\\
&+ 2 \sum_{(n,m) \in \mathcal{E}} \mathbb{E}\left[\langle \bm{\lambda}_{n,m}^{k+1}-\bm{\lambda}_{n,m}^{k}, \bm{\epsilon}_{m}^{k+1}-\bm{\epsilon}_{n}^{k+1} \rangle \right].
\end{align}
\normalsize
On the other hand, we have
\small
\begin{align}
\nonumber &\!2 \rho  \sum_{(n,m) \in \mathcal{E}}\!\mathbb{E}\left[\langle \bm{\epsilon}_m^{k+1}\!-\!\bm{\epsilon}_m^{k} , \bm{r}_{n,m}^{k+1} \rangle \right]\!+\!2 \rho \sum_{n=1}^N \mathbb{E}\left[\langle d_{n} \bm{\epsilon}_n^{k+1} , \bm{\theta}_n^\star\!-\!\bm{\theta}_n^{k+1} \rangle \right] \\ 
\nonumber &\!=\!2 \rho \sum_{(n,m) \in \mathcal{E}}\! \mathbb{E}\left[\langle \bm{\epsilon}_n^{k+1} , \bm{\theta}_m^\star\!-\!\bm{\theta}_m^{k+1} \rangle \right]\!+\!2 \rho \sum_{(n,m) \in \mathcal{E}} \mathbb{E}\left[\langle \bm{\epsilon}_m^{k+1} , \bm{\theta}_m^\star\!-\!\bm{\theta}_m^{k+1} \rangle \right]\\
\nonumber &\!-\!2 \rho \sum_{(n,m) \in \mathcal{E}} \mathbb{E}\left[\| \bm{\epsilon}_n^{k+1}\!-\!\bm{\epsilon}_m^{k+1} \|^2 \right]\!+\!2 \rho  \sum_{(n,m) \in \mathcal{E}}  \mathbb{E}\left[\langle \bm{\lambda}_{n,m}^{k+1}\!-\!\bm{\lambda}_{n,m}^{k}, \bm{\epsilon}_m^{k+1}\!-\!\bm{\epsilon}_n^{k+1} \rangle \right]\\
& \!-\!2 \rho  \sum_{(n,m) \in \mathcal{E}} \mathbb{E}\left[ \langle \bm{\epsilon}_m^{k} , \bm{r}_{n,m}^{k+1} \rangle \right].
\end{align}
\normalsize
Now, recall that $\bm{\theta}_{m}^{k+1}$, $m \in \mathcal{T}$ minimizes the function $f_m(\bm{\theta}_m) + \langle \bm{\alpha}_m^{k+1} + \rho d_{m} \bm{\epsilon}_m^{k+1}, \bm{\theta}_m \rangle$ and $\bm{\theta}_{m}^{k}$, $m \in \mathcal{T}$ minimizes the function $f_m(\bm{\theta}_m) + \langle \bm{\alpha}_m^{k} + \rho d_{m} \bm{\epsilon}_m^{k}, \bm{\theta}_m \rangle$, then we could write
\small
\begin{align}
\nonumber &\mathbb{E}\left[f_m(\bm{\theta}_m^{k+1}) \right]\!+\!\mathbb{E}\left[\langle \bm{\alpha}_m^{k+1}\!+\!\rho d_{m} \bm{\epsilon}_m^{k+1}, \bm{\theta}_m^{k+1} \rangle \right] \\
&\!\leq\! \mathbb{E}\left[f_m(\bm{\theta}_m^{k})\right]\!+\!\mathbb{E}\left[\langle \bm{\alpha}_m^{k+1}\!+\!\rho d_{m} \bm{\epsilon}_m^{k+1}, \bm{\theta}_m^{k} \rangle \right],
\\
\nonumber &\mathbb{E}\left[f_m(\bm{\theta}_m^{k})\right]\!+\!\mathbb{E}\left[\langle \bm{\alpha}_m^{k}\!+\!\rho d_{m} \bm{\epsilon}_m^{k}, \bm{\theta}_m^{k} \rangle \right]\\
 &\!\leq\!\mathbb{E}\left[f_m(\bm{\theta}_m^{k+1})\right]\!+\!\mathbb{E}\left[ \langle \bm{\alpha}_m^{k}\!+\!\rho d_{m} \bm{\epsilon}_m^{k}, \bm{\theta}_m^{k+1} \rangle \right].
\end{align}
\normalsize
Adding both equations and re-arranging the terms, we get
\small
\begin{align}\label{last}
\nonumber &\mathbb{E}\left[\langle \bm{\alpha}_m^{k+1} - \bm{\alpha}_m^{k}, \bm{\theta}_m^{k+1}-\bm{\theta}_m^{k} \rangle \right]\\ &\leq - \rho d_{m} \mathbb{E}\left[\langle \bm{\epsilon}_m^{k+1} - \bm{\epsilon}_m^{k} , \bm{\theta}_m^{k+1}-\bm{\theta}_m^{k} \rangle \right].
\end{align}
\normalsize
Using the update of $\bm{\alpha}_m^{k+1}$, i.e. $
\bm{\alpha}_m^{k+1} = \bm{\alpha}_m^{k} + \rho \sum_{n \in \mathcal{N}_m} \bm{r}_{m,n}^{k+1}$, we can re-write \eqref{last} to get
\small
\begin{align}
\nonumber &\!-\!\rho \sum_{(n,m) \in \mathcal{E}}\!\mathbb{E}\left[\langle \bm{r}_{n,m}^{k+1} ,\bm{\theta}_{m}^{k+1}\!-\!\bm{\theta}_{m}^{k} \rangle \right]\\
&\leq\!-\!\rho \sum_{(n,m) \in \mathcal{E}} \mathbb{E}\left[\langle \bm{\epsilon}_m^{k+1}\!-\!\bm{\epsilon}_m^{k} , \bm{\theta}_m^{k+1}\!-\!\bm{\theta}_m^{k} \rangle \right].
\end{align}
\normalsize
where we used $\bm{r}_{m,n}^{k+1} = -\bm{r}_{n,m}^{k+1}$ after summing over $m \in \mathcal{T}$. Going back to \eqref{eqq1}, we can write
\small
\begin{align}\label{eqq2}
\nonumber &\!\frac{1}{\rho} \sum_{(n,m) \in \mathcal{E}}\! \mathbb{E}\left[\|\bm{\lambda}_{n,m}^{k+1}\!-\!\bm{\lambda}_{n,m}^{\star}\|^2 \right]\!-\!\frac{1}{\rho}\sum_{(n,m) \in \mathcal{E}}\!\mathbb{E}\left[\|\bm{\lambda}_{n,m}^{k}\!-\!\bm{\lambda}_{n,m}^{\star}\|^2 \right]\\
\nonumber &\!+\!\frac{1}{\rho} \sum_{(n,m) \in \mathcal{E}}\!\mathbb{E}\left[\|\bm{\lambda}_{n,m}^{k+1}\!-\!\bm{\lambda}_{n,m}^{k}\|^2 \right]\!+\!\rho  \sum_{(n,m) \in \mathcal{E}} \mathbb{E}\left[\|\bm{\theta}_{m}^{k+1}\!-\!\bm{\theta}_{m}^{\star}\|^2 \right]\\
\nonumber &\!-\!\rho \sum_{(n,m) \in \mathcal{E}} \mathbb{E}\left[\|\bm{\theta}_{m}^{k}\!-\!\bm{\theta}_{m}^{\star}\|^2 \right]\!+\!\rho  \sum_{(n,m) \in \mathcal{E}} \mathbb{E}\left[\|\bm{\theta}_{m}^{k+1}\!-\!\bm{\theta}_{m}^{k}\|^2 \right]\\
\nonumber &\!\leq\!2 \sum_{(n,m) \in \mathcal{E}}\!\mathbb{E}\left[\langle \bm{\lambda}_{n,m}^{k+1}\!-\!\bm{\lambda}_{n,m}^{\star}, \bm{\epsilon}_m^{k+1}\!-\!\bm{\epsilon}_n^{k+1} \rangle \right]\\
\nonumber &\!+\!2 \sum_{(n,m) \in \mathcal{E}} \mathbb{E}\left[\langle \bm{\lambda}_{n,m}^{k+1}\!-\!\bm{\lambda}_{n,m}^{k}, \bm{\epsilon}_m^{k+1}\!-\!\bm{\epsilon}_n^{k+1} \rangle \right] \\
\nonumber &\!-\!2 \rho  \sum_{(n,m) \in \mathcal{E}}  \mathbb{E}\left[ \langle \bm{\theta}_{m}^{k+1}\!-\!\bm{\theta}_{m}^{k}, \bm{\epsilon}_{m}^{k+1}\!-\!\bm{\epsilon}_{m}^{k} \rangle \right]\!-\!2 \rho \sum_{(n,m) \in \mathcal{E}} \mathbb{E}\left[\langle \bm{\epsilon}_m^{k} , \bm{r}_{n,m}^{k+1} \rangle \right]\\
&\!+\!2 \rho \sum_{(n,m) \in \mathcal{E}} \mathbb{E}\left[\langle \bm{\epsilon}_{m}^{k}\!+\!\bm{\epsilon}_n^{k+1} , \bm{\theta}_m^\star\!-\!\bm{\theta}_m^{k+1} \rangle \right] .
\end{align}
\normalsize
To upper bound the terms in the right hand side, we will use the identity \eqref{id2}
\small
\begin{align}
\nonumber &2\!\sum_{(n,m) \in \mathcal{E}} \mathbb{E}\left[\langle \bm{\lambda}_{n,m}^{k+1}\!-\!\bm{\lambda}_{n,m}^{\star}, \bm{\epsilon}_m^{k+1}\!-\!\bm{\epsilon}_n^{k+1} \rangle \right]\\
&\leq\!\frac{1}{\eta_1}\!\sum_{(n,m) \in \mathcal{E}} \mathbb{E}\left[\|\bm{\epsilon}_m^{k+1}\!-\!\bm{\epsilon}_n^{k+1}\|^2 \right]\!+\!\eta_1\!\sum_{(n,m) \in \mathcal{E}} \mathbb{E}\left[\|\bm{\lambda}_{n,m}^{k+1}\!-\!\bm{\lambda}_{n,m}^{\star}\|^2 \right], \label{firstbound}\\
\nonumber &2\!\sum_{(n,m) \in \mathcal{E}} \mathbb{E}\left[\langle \bm{\lambda}_{n,m}^{k+1}\!-\!\bm{\lambda}_{n,m}^{k}, \bm{\epsilon}_m^{k+1}\!-\!\bm{\epsilon}_n^{k+1} \rangle \right]\\
&\leq\!\frac{1}{\eta_2}\!\sum_{(n,m) \in \mathcal{E}} \mathbb{E}\left[\|\bm{\epsilon}_m^{k+1}\!-\!\bm{\epsilon}_n^{k+1}\|^2 \right]\!+\!\eta_2\!\sum_{(n,m) \in \mathcal{E}} \mathbb{E}\left[\|\bm{\lambda}_{n,m}^{k+1}\!-\!\bm{\lambda}_{n,m}^{k}\|^2 \right], \label{secondtbound} \\
\nonumber &2 \rho \sum_{(n,m) \in \mathcal{E}} \mathbb{E}\left[\langle \bm{\epsilon}_m^{k}\!+\!\bm{\epsilon}_n^{k+1}, \bm{\theta}_m^{\star}\!-\!\bm{\theta}_m^{k+1} \rangle \right] \\
&\leq \frac{\rho}{\eta_3} \sum_{(n,m) \in \mathcal{E}} \mathbb{E}\left[\|\bm{\epsilon}_m^{k}\!+\!\bm{\epsilon}_n^{k+1}\|^2 \right]\!+\!\rho \eta_3 \sum_{(n,m) \in \mathcal{E}} \mathbb{E}\left[\|\bm{\theta}_m^{\star}\!-\!\bm{\theta}_m^{k+1}\|^2 \right], \label{thirdtbound}\\
\nonumber &-2 \rho \sum_{(n,m) \in \mathcal{E}} \mathbb{E}\left[\langle \bm{\theta}_m^{k+1}\!-\!\bm{\theta}_m^{k}, \bm{\epsilon}_m^{k+1}\!-\!\bm{\epsilon}_m^{k} \rangle \right] \\
&\leq \frac{\rho}{\eta_4} \sum_{(n,m) \in \mathcal{E}} \mathbb{E}\left[\|\bm{\epsilon}_m^{k+1}\!-\!\bm{\epsilon}_m^{k}\|^2 \right]\!+\!\rho \eta_4 \sum_{(n,m) \in \mathcal{E}} \mathbb{E}\left[\|\bm{\theta}_m^{k+1}\!-\!\bm{\theta}_m^{k}\|^2 \right], \label{fourthtbound}
\end{align}
\normalsize
Finally, we use both identities \eqref{id1} and \eqref{id2} to get the following bound
\small
\begin{align}\label{fifthtbound}
\nonumber &-2 \rho \sum_{(n,m) \in \mathcal{E}} \mathbb{E}\left[\langle \bm{\epsilon}_m^{k}, \bm{r}_{n,m}^{k+1} \rangle \right] \\
\nonumber&\!\leq\!\frac{\rho}{\eta_5} \sum_{(n,m) \in \mathcal{E}} \mathbb{E}\left[\|\bm{\epsilon}_m^{k}\|^2 \right] 
\!+\!\frac{2 \rho}{\eta_5} \sum_{(n,m) \in \mathcal{E}} \mathbb{E}\left[\|\bm{\lambda}_{n,m}^{k+1}\!-\!\bm{\lambda}_{n,m}^{k}\|^2 \right]\\
&+\!2 \rho \eta_5  \sum_{(n,m) \in \mathcal{E}} \mathbb{E}\left[\|\bm{\epsilon}_m^{k+1}\!-\!\bm{\epsilon}_n^{k+1}\|^2 \right],
\end{align}
\normalsize
where $\{\eta_i\}_{i=1}^5$ are are arbitrary positive constants to be specified later on. Using these bounds and re-arranging the terms in \eqref{eqq2}, we can write
\small
\begin{align}
\nonumber & 2 \rho \sum_{(n,m) \in \mathcal{E}} \mathbb{E}\left[\|\bm{\epsilon}_n^{k+1} - \bm{\epsilon}_m^{k+1}\|^2 \right] + \rho (1-\eta_4) \sum_{(n,m) \in \mathcal{E}} \mathbb{E}\left[\|\bm{\theta}_m^{k+1} - \bm{\theta}_m^{k}\|^2 \right] \\
\nonumber & + \left(\frac{1-2 \eta_5}{\rho} - \eta_2 \right) \sum_{(n,m) \in \mathcal{E}} \mathbb{E}\left[\|\bm{\lambda}_{n,m}^{k+1} - \bm{\lambda}_{n,m}^{k}\|^2 \right]\\
\nonumber & \leq \left(\frac{1}{\eta_1} + \frac{1}{\eta_2} + 2 \rho \eta_5 \right) \sum_{(n,m) \in \mathcal{E}} \mathbb{E}\left[\|\bm{\epsilon}_n^{k+1} - \bm{\epsilon}_m^{k+1}\|^2 \right]\\
\nonumber &  + \frac{\rho}{\eta_3} \sum_{(n,m) \in \mathcal{E}} \mathbb{E}\left[\|\bm{\epsilon}_n^{k+1} + \bm{\epsilon}_m^{k}\|^2 \right]+ \frac{1}{\rho} \sum_{(n,m) \in \mathcal{E}} \mathbb{E}\left[\|\bm{\lambda}_{n,m}^{k} - \bm{\lambda}_{n,m}^{\star}\|^2 \right]  \\
\nonumber &  - \left(\frac{1}{\rho} - \eta_1 \right) \sum_{(n,m) \in \mathcal{E}} \mathbb{E}\left[\|\bm{\lambda}_{n,m}^{k+1} - \bm{\lambda}_{n,m}^{\star}\|^2 \right] + \rho \sum_{(n,m) \in \mathcal{E}} \mathbb{E}\left[\|\bm{\theta}_m^{k} - \bm{\theta}_m^{\star}\|^2 \right] \\
\nonumber &\!-\!\rho (1-\eta_3) \sum_{(n,m) \in \mathcal{E}} \mathbb{E}\left[\|\bm{\theta}_m^{k+1}\!-\!\bm{\theta}_m^{\star}\|^2 \right]\!+\!\frac{\rho}{\eta_5} \sum_{(n,m) \in \mathcal{E}} \mathbb{E}\left[\|\bm{\epsilon}_m^{k}\|^2\right]\\
&\!+\!\frac{\rho}{\eta_4} \sum_{(n,m) \in \mathcal{E}} \mathbb{E}\left[\|\bm{\epsilon}_m^{k+1}\!-\!\bm{\epsilon}_m^{k}\|^2 \right].
\end{align}
\normalsize
Fix the values of $\{\eta_i\}_{i=1}^5$ to be $(\eta_1, \eta_2, \eta_3, \eta_4, \eta_5) = \left(\frac{\psi^k}{2 \psi^0 \rho}, \frac{1}{4 \rho}, \frac{\psi^k}{2 \psi^0}, \frac{1}{2}, \frac{1}{4} \right)$, we get
\small
\begin{align}
\nonumber & \frac{\rho}{2} \sum_{(n,m) \in \mathcal{E}} \mathbb{E}\left[\|\bm{\theta}_m^{k+1} - \bm{\theta}_m^{k}\|^2 \right] + \frac{1}{4 \rho} \sum_{(n,m) \in \mathcal{E}} \mathbb{E}\left[\|\bm{\lambda}_{n,m}^{k+1} - \bm{\lambda}_{n,m}^{k}\|^2 \right]\\
\nonumber & \leq \left(\frac{5 \rho}{2} + \frac{2 \rho \psi^0}{\psi^k} \right) \sum_{(n,m) \in \mathcal{E}} \mathbb{E}\left[\|\bm{\epsilon}_n^{k+1} - \bm{\epsilon}_m^{k+1}\|^2 \right] \\
\nonumber &\!+\!\frac{2 \rho \psi^0}{\psi^k} \sum_{(n,m) \in \mathcal{E}}\!\mathbb{E}\left[\|\bm{\epsilon}_n^{k+1}\!+\!\bm{\epsilon}_m^{k}\|^2 \right]\!+\!\frac{1}{\rho} \sum_{(n,m) \in \mathcal{E}}\!\mathbb{E}\left[\|\bm{\lambda}_{n,m}^{k}\!-\!\bm{\lambda}_{n,m}^{\star}\|^2 \right] \\
\nonumber &  - \frac{1}{\rho} \left(1 - \frac{\psi^k}{2 \psi^0}\right) \sum_{(n,m) \in \mathcal{E}} \mathbb{E}\left[\|\bm{\lambda}_{n,m}^{k+1} - \bm{\lambda}_{n,m}^{\star}\|^2\right]  \\
\nonumber&\!+\!\rho \sum_{(n,m) \in \mathcal{E}}\!\mathbb{E}\left[\|\bm{\theta}_m^{k}\!-\!\bm{\theta}_m^{\star}\|^2 \right]\!-\!\rho (1\!-\!\frac{\psi^k}{2 \psi^0}) \sum_{(n,m) \in \mathcal{E}}\!\mathbb{E}\left[\|\bm{\theta}_m^{k+1}\!-\!\bm{\theta}_m^{\star}\|^2 \right]\\
&\!+\!4 \rho \sum_{(n,m) \in \mathcal{E}}\!\mathbb{E}\left[\|\bm{\epsilon}_m^{k}\|^2 \right]\!+\!2 \rho \sum_{(n,m) \in \mathcal{E}}\! \mathbb{E}\left[\|\bm{\epsilon}_m^{k+1}\!-\!\bm{\epsilon}_m^{k}\|^2 \right]
\end{align}
\normalsize
Re-arranging the terms, we can write
\small
\begin{align}
\nonumber & \frac{\rho}{2} \sum_{(n,m) \in \mathcal{E}} \mathbb{E}\left[\|\bm{\theta}_m^{k+1} - \bm{\theta}_m^{k}\|^2 \right] + \frac{1}{4 \rho} \sum_{(n,m) \in \mathcal{E}} \mathbb{E}\left[\|\bm{\lambda}_{n,m}^{k+1} - \bm{\lambda}_{n,m}^{k}\|^2 \right]\\
\nonumber &\!\leq\!\frac{1}{\rho} \sum_{(n,m) \in \mathcal{E}}\!\mathbb{E}\left[\|\bm{\lambda}_{n,m}^{k}\!-\! \bm{\lambda}_{n,m}^{\star}\|^2 \right]\!-\!\frac{1}{\rho} \left(1\!-\!\frac{\psi^k}{2 \psi^0}\right)\!\mathbb{E}\left[\|\bm{\lambda}_{n,m}^{k+1}\!-\!\bm{\lambda}_{n,m}^{\star}\|^2 \right]\\
\nonumber &\!+\!\rho \sum_{(n,m) \in \mathcal{E}}\!\mathbb{E}\left[\|\bm{\theta}_m^{k}\!-\!\bm{\theta}_m^{\star}\|^2 \right]\!-\!\rho \left(1\!-\!\frac{\psi^k}{2 \psi^0}\right) \sum_{(n,m) \in \mathcal{E}} \mathbb{E}\left[\|\bm{\theta}_m^{k+1}\!-\!\bm{\theta}_m^{\star}\|^2 \right] \\
\nonumber &\!+\!\left(5 \rho\!+\!\frac{8 \rho \psi^0}{\psi^k} \right) \sum_{(n,m) \in \mathcal{E}}\!\mathbb{E}\left[\|\bm{\epsilon}_n^{k+1}\|^2 \right]\!+\!\left(9 \rho\!+\!\frac{4 \rho \psi^0}{\psi^k} \right) \sum_{(n,m) \in \mathcal{E}}\!\mathbb{E}\left[\|\bm{\epsilon}_m^{k+1}\|^2 \right] \\
&  + \left(8 \rho + \frac{4 \rho \psi^0}{\psi^k} \right) \sum_{(n,m) \in \mathcal{E}} \mathbb{E}\left[\|\bm{\epsilon}_m^{k}\|^2 \right].
\end{align}
\normalsize
Therefore, using \eqref{err_bound}, we can write
\small
\begin{align}
\nonumber & \frac{\rho}{2} \sum_{(n,m) \in \mathcal{E}} \mathbb{E}\left[\|\bm{\theta}_m^{k+1} - \bm{\theta}_m^{k}\|^2 \right] + \frac{1}{4 \rho} \sum_{(n,m) \in \mathcal{E}} \mathbb{E}\left[\|\bm{\lambda}_{n,m}^{k+1} - \bm{\lambda}_{n,m}^{k}\|^2 \right]\\
\nonumber & \leq \frac{1}{\rho} \sum_{(n,m) \in \mathcal{E}} \mathbb{E}\left[\|\bm{\lambda}_{n,m}^{k} - \bm{\lambda}_{n,m}^{\star}\|^2 \right]\\
\nonumber &\!-\!\frac{1}{\rho} \left(1\!-\!\frac{\psi^k}{2 \psi^0}\right)\!\mathbb{E}\left[\|\bm{\lambda}_{n,m}^{k+1}\!-\!\bm{\lambda}_{n,m}^{\star}\|^2 \right]\!+\!\rho \sum_{(n,m) \in \mathcal{E}} \mathbb{E}\left[\|\bm{\theta}_m^{k}\!-\! \bm{\theta}_m^{\star}\|^2 \right] \\
&\!-\!\rho \left(1\!-\!\frac{\psi^k}{2 \psi^0}\right) \sum_{(n,m) \in \mathcal{E}}\!\mathbb{E}\left[\|\bm{\theta}_m^{k+1}\!-\!\bm{\theta}_m^{\star}\|^2 \right]\!+\!\gamma_1 \psi^k\!+\!\gamma_2 \psi^{2k},
\end{align}
\normalsize
where $\gamma_1 = 64 \rho C_0 \psi^0 |\mathcal{E}|$ and $\gamma_2 = 88 \rho C_0^2 |\mathcal{E}|$. Now, we define the Lyapunov function 
\small
\begin{align}
V^k\!=\!\frac{1}{\rho} \sum_{(n,m) \in \mathcal{E}}\!\|\bm{\lambda}_{n,m}^{k}\!-\!\bm{\lambda}_{n,m}^{\star}\|^2\!+\! \rho \sum_{(n,m) \in \mathcal{E}}\!\|\bm{\theta}_m^{k}\!-\!\bm{\theta}_m^{\star}\|^2.
\end{align}
\normalsize
Thus, we get
\small
\begin{align}\label{eqly}
\nonumber &\frac{\rho}{2} \sum_{(n,m) \in \mathcal{E}}\!\mathbb{E}\left[\|\bm{\theta}_m^{k+1}\!-\!\bm{\theta}_m^{k}\|^2 \right]\!+\!\frac{1}{4 \rho} \sum_{(n,m) \in \mathcal{E}} \mathbb{E}\left[\|\bm{\lambda}_{n,m}^{k+1}\!-\!\bm{\lambda}_{n,m}^{k}\|^2 \right]\\
&\leq\!\mathbb{E}\left[V^k\right]\!-\!(1\!-\! \frac{\psi^k}{2 \psi^0}) \mathbb{E}\left[V^{k+1}\right]\!+\! \gamma_1 \psi^k\!+\!\gamma_2 \psi^{2k}.
\end{align}
\normalsize
As a consequence, we can write that $\mathbb{E}\left[V^{k+1}\right] \leq \left(1 - \frac{\psi^k}{2 \psi^0}\right)^{-1} \left(\mathbb{E}\left[V^k\right] + \gamma_1 \psi^k + \gamma_2 \psi^{2k} \right)$. Using this equation iteratively, we obtain
\small
\begin{align}
\nonumber &\mathbb{E}\left[V^{k+1}\right] \\
\nonumber & \leq \left(\prod_{j=0}^{k} \left(1\!-\!\frac{\psi^j}{2 \psi^0}\right)^{-1}\right) \mathbb{E}\left[V_0\right]\!+\!\gamma_1 \sum_{j=0}^{k} \prod_{i=j}^{k} \left(1\!-\!\frac{\psi^i}{2 \psi^0}\right)^{-1} \psi^j\!\\
\nonumber &+\!\gamma_2 \sum_{j=0}^{k} \prod_{i=j}^{k} \left(1\!-\!\frac{\psi^i}{2 \psi^0}\right)^{-1} \psi^{2j} \\
& \leq \prod_{j=0}^{\infty} \left(1\!-\!\frac{\psi^j}{2 \psi^0}\right)^{-1} \left( \mathbb{E}\left[V^0\right]\!+\!\gamma_1 \sum_{i=0}^{\infty} \psi^{i}\!+\!\gamma_2 \sum_{i=0}^{\infty} \psi^{2i} \right).
\end{align}
\normalsize
where we have used the fact that $\left(1 - \frac{\psi^k}{2 \psi^0}\right) \in [\frac{1}{2}, 1 ]$. Since $\sum_{i=0}^{\infty} \omega^{i} < \infty$ and $\sum_{i=0}^{\infty} \xi^{i} < \infty$, thus $\sum_{i=0}^{\infty} \psi^{i} < \infty$. Furthermore, the sequence $\{\psi^i\}$ is non-negative, then we get that $\sum_{i=0}^{\infty} \psi^{2i} < \infty$. To show that $\prod_{j=0}^{\infty} \left(1 - \frac{\psi^j}{2 \psi^0}\right)^{-1}$ is also finite, we consider its logarithm, i.e.
\small
\begin{align}
& \nonumber \log\left(\prod_{j=0}^{\infty} \left(1\!-\!\frac{\psi^j}{2 \psi^0}\right)^{-1}\right) \overset{\mathrm{(a)}}{\leq} \sum_{j=0}^{\infty}\log\left(\left(1\!-\!\frac{\psi^j}{2 \psi^0}\right)^{-1}\right) \\ & \overset{\mathrm{(b)}}{\leq} \sum_{j=0}^{\infty}\log\left(1\!+\!\frac{\psi^j}{\psi^0}\right)\leq \frac{1}{\psi^0}\sum_{j=0}^{\infty} \psi^j,
\end{align}
\normalsize
where we have used that $ \log\left(\left(1 - \frac{z}{2}\right)^{-1}\right) \leq \log(1+z),~z \geq 1$ in $\mathrm{(a)}$ and $\log(1+z) \leq z,~z \geq 1$ in $\mathrm{(b)}$. Hence, $\prod_{j=0}^{\infty} \left(1 - \frac{\psi^j}{2 \psi^0}\right)^{-1}$ is also finite and we conclude that the sequence $\mathbb{E}\left[V^k\right]$ is upper bounded by a finite quantity that we denote as $\bar{V}$. Going back to \eqref{eqly} and taking the sum from $k=0$ to $\infty$ while using the upper bound on $\mathbb{E}\left[V^k\right]$, we get
\small
\begin{align}
\nonumber & \sum_{k=0}^{\infty}  \sum_{(n,m) \in \mathcal{E}}\left\{\frac{\rho}{2} \mathbb{E}\left[\|\bm{\theta}_m^{k+1}\!-\!\bm{\theta}_m^{k}\|^2 \right] + \!\frac{1}{4 \rho}\!\mathbb{E}\left[\|\bm{\lambda}_{n,m}^{k+1}\!-\!\bm{\lambda}_{n,m}^{k}\|^2 \right]\right\} \\
&\leq V^0\!+\!(\frac{\bar{V}}{2 \psi^0}\!+\!\gamma_1) \sum_{k=0}^{\infty}  \psi^k\!+\!\gamma_2 \sum_{k=0}^{\infty} \psi^{2k}.
\end{align}
\normalsize
Since the right hand side is finite, we conclude that the left hand side is convergent and that
\small
\begin{align}
&\underset{k \rightarrow \infty}{\lim} \mathbb{E}\left[\|\bm{\theta}_m^{k+1} - \bm{\theta}_m^{k}\|^2 \right] = 0 , \label{cvtheta}\\
&\underset{k \rightarrow \infty}{\lim} \mathbb{E}\left[\|\bm{\lambda}_{n,m}^{k+1} - \bm{\lambda}_{n,m}^{k}\|^2 \right] = 0. \label{cvlamda}
\end{align}
\normalsize
Using the definition of the primal and dual residuals and \eqref{id1}, we can write
\small
\begin{align}
\nonumber &\mathbb{E}\left[\|\bm{r}_{n,m}^{k+1}\|^2 \right]\\
&\!\leq\!2 \left(\frac{1}{\rho^2} \mathbb{E}\left[\|\bm{\lambda}_{n,m}^{k+1}\!-\!\bm{\lambda}_{n,m}^{k}\|^2 \right]\!+\!2 \left( \mathbb{E}\left[\|\bm{\epsilon}_m^k \|^2 \right]\!+\!\mathbb{E}\left[\|\bm{\epsilon}_m^{k+1} \|^2 \right]\right) \right), \\
\nonumber &\mathbb{E}\left[\|\bm{s}_{n}^{k+1}\|^2 \right]\\ &\!\leq\!2 \rho^2 d_{n} \left(\sum_{m \in \mathcal{N}_n}\! \mathbb{E}\left[\|\bm{\theta}_{m}^{k+1}\!-\!\bm{\theta}_{m}^{k}\|^2 \right]\!+\!\sum_{m \in \mathcal{N}_n} \mathbb{E}\left[\|\bm{\epsilon}_{m}^{k}\!-\!\bm{\epsilon}_{m}^{k+1}\|^2 \right] \right).
\end{align}
\normalsize
Since $\mathbb{E}\left[\|\bm{\epsilon}_n^k \|^2 \right] \leq 4 C_0^2 \psi^{2k}, ~\forall n$, then $\underset{k \rightarrow \infty}{\lim} \mathbb{E}\left[\|\bm{\epsilon}_n^k \|^2 \right]  = 0$. Using \eqref{cvtheta}, \eqref{cvlamda} and $\underset{k \rightarrow \infty}{\lim} \mathbb{E}\left[\|\bm{\epsilon}_n^k \|^2 \right]  = 0$, we conclude that $\underset{k \rightarrow \infty}{\lim} \mathbb{E}\left[\|\bm{r}_{n,m}^{k+1}\|^2 \right] = 0$ and $\underset{k \rightarrow \infty}{\lim} \mathbb{E}\left[\|\bm{s}_{n}^{k+1}\|^2 \right] = 0$. 

\noindent Using the Cauchy–Schwarz inequality, we can write
\small
\begin{align}
& \left|\mathbb{E}\left[\langle \bm{\lambda}_{n,m}^{k+1}, \bm{r}_{n,m}^{k+1}\rangle \right]\right|^2 \leq  \mathbb{E}\left[\|\bm{\lambda}_{n,m}^{k+1}\|^2 \right]\mathbb{E}\left[\|\bm{r}_{n,m}^{k+1}\|^2 \right],\\ \label{sol1}
& \left|\mathbb{E}\left[\langle \bm{\lambda}_{n,m}^{\star}, \bm{r}_{n,m}^{k+1}\rangle \right]\right|^2 \leq  \mathbb{E}\left[\|\bm{\lambda}_{n,m}^{\star}\|^2 \right] \mathbb{E}\left[\|\bm{r}_{n,m}^{k+1}\|^2 \right],\\
& \left|\mathbb{E}\left[\langle \bm{s}_{n}^{k+1}, \bm{\theta}_{n}^\star\!-\!\bm{\theta}_{n}^{k+1}\rangle \right]\right|^2 \leq \mathbb{E}\left[\|\bm{s}_{n}^{k+1}\|^2 \right]\mathbb{E}\left[\|\bm{\theta}_{n}^\star\!-\!\bm{\theta}_{n}^{k+1}\|^2 \right],\\
& \left|\mathbb{E}\left[\langle \bm{\epsilon}_n^{k+1} , \bm{\theta}_n^\star - \bm{\theta}_n^{k+1} \rangle \right]\right|^2 \leq \mathbb{E}\left[\|\bm{\epsilon}_n^{k+1}\|^2 \right] \mathbb{E}\left[\|\bm{\theta}_n^\star - \bm{\theta}_n^{k+1}\|^2 \right]. \label{sol2}
\end{align}
\normalsize
Since $\underset{k \rightarrow \infty}{\lim} \mathbb{E}\left[\|\bm{\epsilon}_n^k \|^2 \right]  = 0$, $\underset{k \rightarrow \infty}{\lim} \mathbb{E}\left[\|\bm{r}_{n,m}^{k+1}\|^2 \right] = 0$ and $\underset{k \rightarrow \infty}{\lim} \mathbb{E}\left[\|\bm{s}_{n}^{k+1}\|^2 \right] = 0$, we get that
\small
\begin{align}
&\underset{k \rightarrow \infty}{\lim} \mathbb{E}\left[\langle \bm{\lambda}_{n,m}^{k+1}, \bm{r}_{n,m}^{k+1}\rangle \right] = 0, \\
&\underset{k \rightarrow \infty}{\lim} \mathbb{E}\left[\langle \bm{\lambda}_{n,m}^{\star}, \bm{r}_{n,m}^{k+1}\rangle \right] = 0, \\
&\underset{k \rightarrow \infty}{\lim} \mathbb{E}\left[\langle \bm{s}_{n}^{k+1}, \bm{\theta}_{n}^\star-\bm{\theta}_{n}^{k+1}\rangle \right] = 0, \\
&\underset{k \rightarrow \infty}{\lim} \mathbb{E}\left[\langle \bm{\epsilon}_n^{k+1} , \bm{\theta}_n^\star - \bm{\theta}_n^{k+1} \rangle \right] = 0.
\end{align}
\normalsize
From $(i)$ and $(ii)$ of Lemma \ref{lemma1}, we conclude that $\underset{k \rightarrow \infty}{\lim} \sum_{n=1}^N \mathbb{E}\left[f_n(\bm{\theta}_{n}^{k}) - f_n(\bm{\theta}_{n}^\star)\right] = 0$.
\subsection{Proof of Theorem \ref{thm2}}\label{proofthm2}
The proof of Theorem \ref{thm2} follows similar steps as the proof of convergence rate of \cite{liu2019}. However, the alternating update nature of our algorithm makes the updates happen in an asymmetric manner, in contrast to the symmetric update in \cite{liu2019}, which makes the proof more complex. For a bipartite graph, the adjacency matrix can be written as $\bm{A} = [\bm{0}_{rr} , \bm{B}; \bm{B}^T  , \bm{0}_{ss}]$ where $r = |\mathcal{H}|$, $s = |\mathcal{T}|$ are the cardinalities of $\mathcal{H}$ and $\mathcal{T}$, respectively. The matrices $\bm{0}_{rr}$, and $\bm{0}_{ss}$ are the null matrices of order $r\times r$, and $s\times s$, respectively. The matrix $\bm{B} \in \mathbb{R}^{r\times s}$ is called the bi-adjacency matrix. The adjacency matrix is a boolean matrix where each element is defined as $ \bm{A}_{i,j} = 1$ if there exists a link between the nodes $i$ and $j$ (i.e. workers), otherwise $ \bm{A}_{i,j} = 0$. In our analysis, we also define the matrix $\bm{C} = [\bm{0}_{rr} , \bm{B}; \bm{0}_{rs}  , \bm{0}_{ss}]$ as well as the matrices $
\bm{\theta} = [\bm{\theta}_1, \dots, \bm{\theta}_N]^T$, $ \bm{\alpha} = [\bm{\alpha}_1, \dots, \bm{\alpha}_N]^T$, $\bm{\hat{\theta}} = [\bm{\hat{\theta}}_1,\dots,\bm{\hat{\theta}}_N]^T$, and $\bm{E} = [\bm{\epsilon}_1, \dots, \bm{\epsilon}_N]^T$.
In this section, we introduce matrices related to the network topology: the diagonal degree matrix $\bm{D}$,  the signed and unsigned incidence matrices, $\bm{M}_{-}$ and $\bm{M}_{+}$, respectively. Using \eqref{alphaupdate}, \eqref{head_eq}, and \eqref{tail_eq}, we can write
\small
\begin{align}
& \nabla f(\bm{\theta}^{k+1}) + \bm{\alpha}^k - \rho \bm{C} \bm{\hat{\theta}}^{k} - \rho \bm{C}^T \bm{\hat{\theta}}^{k+1} + \rho \bm{D} \bm{\theta}^{k+1} = \bm{0}, \label{equa1} \\
&\bm{\alpha}^{k+1}  = \bm{\alpha}^k + \rho (\bm{D}-\bm{A}) \bm{\hat{\theta}}^{k+1}, \label{equa2}
\end{align}
\normalsize
and the optimality conditions are given by $
\nabla f(\bm{\theta}^\star) + \bm{\alpha}^\star = \bm{0}$, and $\bm{M}_{-}^T \bm{\theta}^\star  = \bm{0}$. Since we have $\bm{D}-\bm{A} = \frac{1}{2} \bm{M}_{-} \bm{M}_{-}^T$, then we can re-write \eqref{equa2} as $
\bm{\alpha}^{k+1}  = \bm{\alpha}^k + \frac{\rho}{2} \bm{M}_{-} \bm{M}_{-}^T \bm{\theta}^{k+1} + \frac{\rho}{2} \bm{M}_{-} \bm{M}_{-}^T \bm{E}^{k+1}$. Initializing $\bm{\alpha}^{0}$ in the column space of $\bm{M}_{-}$, we get that $\bm{\alpha}^{k}$ always stays in the column space of $\bm{M}_{-}$ and thus, we have $\bm{\alpha}^{k} =\bm{M}_{-} \bm{\beta}^k$, $\forall k \geq 0$. Thus, we can further write \eqref{equa2} as
\small
\begin{align}\label{eqbeta}
\bm{\beta}^{k+1}  = \bm{\beta}^k + \frac{\rho}{2} \bm{M}_{-}^T \bm{\theta}^{k+1} + \frac{\rho}{2} \bm{M}_{-}^T \bm{E}^{k+1}.
\end{align}
\normalsize
Using $\bm{D}\!=\!\frac{1}{4} \bm{M}_{-} \bm{M}_{-}^T\!+\! \frac{1}{4} \bm{M}_{+} \bm{M}_{+}^T$, $\bm{A}\!=\!\frac{1}{4} \bm{M}_{+} \bm{M}_{+}^T\!-\!\frac{1}{4} \bm{M}_{-} \bm{M}_{-}^T$ and \eqref{eqbeta}, we can write \eqref{equa1} as
\small
\begin{align}
\nonumber &\nabla\!f(\bm{\theta}^{k+1})\!+\!\bm{M}_{-} \bm{\beta}^{k+1}\!+\!\rho \bm{C} (\bm{E}^{k}\!-\!\bm{\theta}^{k}) + \rho (\bm{C}^T\!-\!\frac{1}{2} \bm{M}_{-} \bm{M}_{-}^T) \bm{E}^{k+1}\\
&+\!\rho\!\left(\bm{A}\!-\!\bm{C}^T\right)\!\bm{\theta}^{k+1}\!=\!\bm{0}.
\end{align}
\normalsize
Using that $\nabla f(\bm{\theta}^\star) + \bm{M}_{-} \bm{\beta}^\star = \bm{0}$ and $\bm{A} = \bm{C} + \bm{C}^T$, we can write
\small
\begin{align}\label{eqbet}
\nonumber& \nabla\!f(\bm{\theta}^{k+1})\!-\!\nabla\!f(\bm{\theta}^\star)\\
\nonumber&=\!\bm{M}_{-} (\bm{\beta}^{\star}\!-\!\bm{\beta}^{k+1}) + \rho \bm{C} (\bm{\theta}^{k}\!-\!\bm{\theta}^{k+1})\!-\!\rho \bm{C} \bm{E}^{k}\\
&+\!\rho \left(\frac{1}{2} \bm{M}_{-} \bm{M}_{-}^T\!-\!\bm{C}^T\right) \bm{E}^{k+1},
\end{align}
\normalsize
then, multiplying both sides by $\bm{\theta}^{k+1}-\bm{\theta}^\star$, we get
\small
\begin{align}\label{eqq_conv}
\nonumber & \mathbb{E} \left[\langle \nabla f(\bm{\theta}^{k+1}) - \nabla f(\bm{\theta}^\star), \bm{\theta}^{k+1}-\bm{\theta}^\star \rangle \right]\\
\nonumber &= \mathbb{E} \left[\langle \bm{M}_{-} (\bm{\beta}^{\star}-\bm{\beta}^{k+1}),  \bm{\theta}^{k+1}-\bm{\theta}^\star \rangle \right]\\
\nonumber & + \rho \mathbb{E} \left[\langle \bm{C} \left(\bm{\theta}^{k} - \bm{\theta}^{k+1}\right), \bm{\theta}^{k+1}-\bm{\theta}^\star \rangle \right]  - \rho \mathbb{E} \left[\langle \bm{C} \bm{E}^{k} , \bm{\theta}^{k+1}-\bm{\theta}^\star \rangle \right]\\
&\!-\!\rho \mathbb{E}\left[\langle \bm{C}^T  \bm{E}^{k+1} , \bm{\theta}^{k+1}\!-\!\bm{\theta}^\star \rangle \right]\!+\! \frac{\rho}{2} \mathbb{E} \left[\langle \bm{M}_{-} \bm{M}_{-}^T \bm{E}^{k+1} , \bm{\theta}^{k+1}\!-\!\bm{\theta}^\star \rangle \right].
\end{align}
\normalsize
The first term of the right hand side can be re-written as
\small
\begin{align}\label{eqq_a}
\nonumber &\mathbb{E}\left[\langle \bm{M}_{-} (\bm{\beta}^{\star}-\bm{\beta}^{k+1}),  \bm{\theta}^{k+1}-\bm{\theta}^\star \rangle \right] \overset{\mathrm{(a)}}{=} \mathbb{E}\left[\langle \bm{\beta}^{\star}-\bm{\beta}^{k+1}, \bm{M}_{-}^T \bm{\theta}^{k+1} \rangle \right] \\
&\overset{\mathrm{(b)}}{=} -\frac{2}{\rho} \mathbb{E}\left[\langle \bm{\beta}^{k+1}\!-\!\bm{\beta}^{\star}, \bm{\beta}^{k+1}\!-\!\bm{\beta}^{k} \rangle \right]\!-\!\mathbb{E}\left[\langle \bm{\beta}^{\star}\!-\!\bm{\beta}^{k+1}, \bm{M}_{-}^T \bm{E}^{k+1} \rangle \right],
\end{align}
\normalsize
where we have used $\bm{M}_{-}^T \bm{\theta}^\star = \bm{0}$ in $\mathrm{(a)}$ and $\bm{M}_{-}^T \bm{\theta}^{k+1} = \frac{2}{\rho} \left( \bm{\beta}^{k+1}-\bm{\beta}^{k+1} \right) - \bm{M}_{-}^T \bm{E}^{k+1}$ in $\mathrm{(b)}$.\\
Expanding the first term of \eqref{eqq_a}, we can write
\small
\begin{align}\label{eqq_b}
\nonumber& -\frac{2 \mathbb{E}\left[\langle \bm{\beta}^{k+1}\!-\!\bm{\beta}^{\star}, \bm{\beta}^{k+1}\!-\!\bm{\beta}^{k} \rangle \right]}{\rho} \\
&=\!\frac{1}{\rho}(\mathbb{E}\left[\|\bm{\beta}^{k}\!-\!\bm{\beta}^{\star}\|_F^2 \right] \!-\!\mathbb{E}\left[\|\bm{\beta}^{k+1}\!-\!\bm{\beta}^{\star}\|_F^2 \right]\!-\!\mathbb{E}\left[\|\bm{\beta}^{k+1}\!-\!\bm{\beta}^{k}\|_F^2 \right]).
\end{align}
\normalsize
Replacing the terms derived in \eqref{eqq_a} and \eqref{eqq_b} by their expressions in \eqref{eqq_conv}, we obtain
\small
\begin{align}\label{eqq_conv_2}
\nonumber & \mathbb{E}\left[\langle \nabla f(\bm{\theta}^{k+1}) - \nabla f(\bm{\theta}^\star), \bm{\theta}^{k+1}-\bm{\theta}^\star \rangle \right]\\
\nonumber &= \frac{1}{\rho} \mathbb{E}\left[\|\bm{\beta}^{k}-\bm{\beta}^{\star}\|_F^2\right] - \frac{1}{\rho} \mathbb{E}\left[\|\bm{\beta}^{k+1}-\bm{\beta}^{\star}\|_F^2\right]\\
\nonumber& - \frac{1}{\rho} \mathbb{E}\left[\|\bm{\beta}^{k+1}-\bm{\beta}^{k}\|_F^2\right] + \mathbb{E}\left[\langle \bm{\beta}^{k+1}-\bm{\beta}^{\star}, \bm{M}_{-}^T \bm{E}^{k+1} \rangle \right]\\
\nonumber &  - \rho \mathbb{E}\left[\langle \bm{C} \bm{E}^{k} , \bm{\theta}^{k+1}-\bm{\theta}^\star \rangle \right] - \rho \mathbb{E}\left[\langle \bm{C}^T  \bm{E}^{k+1} , \bm{\theta}^{k+1}-\bm{\theta}^\star \rangle \right] \\
\nonumber & + \frac{\rho}{2} \mathbb{E}\left[\langle \bm{M}_{-} \bm{M}_{-}^T \bm{E}^{k+1} , \bm{\theta}^{k+1}-\bm{\theta}^\star \rangle \right]\\
&- \rho \mathbb{E}\left[\langle \bm{C} \left(\bm{\theta}^{k+1}- \bm{\theta}^{k}\right), \bm{\theta}^{k+1}-\bm{\theta}^\star \rangle \right].
\end{align}
\normalsize
Using the strong convexity of the function $f$, we can lower bound the left hand side of \eqref{eqq_conv}
\small
\begin{align}
\mathbb{E}\left[\langle \nabla f(\bm{\theta}^{k+1}) - \nabla f(\bm{\theta}^\star), \bm{\theta}^{k+1}-\bm{\theta}^\star \rangle \right] \geq \mu \mathbb{E}\left[\|\bm{\theta}^{k+1}-\bm{\theta}^\star\|_F^2 \right].
\end{align}
\normalsize
Hence, we can write
\small
\begin{align}\label{eqbound}
\nonumber & \frac{1}{\rho} \mathbb{E}\left[\|\bm{\beta}^{k+1}-\bm{\beta}^{\star}\|_F^2 + \mu \|\bm{\theta}^{k+1}-\bm{\theta}^\star\|_F^2 \right]\\
\nonumber& \leq \frac{1}{\rho} \mathbb{E}\left[\|\bm{\beta}^{k}-\bm{\beta}^{\star}\|_F^2 \right] + \rho \mathbb{E}\left[\langle \bm{C} \left(\bm{\theta}^{k}- \bm{\theta}^{\star}\right), \bm{\theta}^{k+1}-\bm{\theta}^\star \rangle \right] \\
\nonumber& +  \rho \mathbb{E}\left[\langle \bm{C} \left(\bm{\theta}^{\star}- \bm{\theta}^{k+1}\right), \bm{\theta}^{k+1}-\bm{\theta}^\star \rangle \right] - \frac{1}{\rho} \mathbb{E}\left[\|\bm{\beta}^{k+1}-\bm{\beta}^{k}\|_F^2 \right]\\
\nonumber& + \mathbb{E}\left[\langle \bm{\beta}^{k+1}-\bm{\beta}^{\star}, \bm{M}_{-}^T \bm{E}^{k+1} \rangle \right] - \rho \mathbb{E}\left[\langle \bm{C} \bm{E}^{k} , \bm{\theta}^{k+1}-\bm{\theta}^\star \rangle \right] \\
& - \rho \mathbb{E}\left[\langle \bm{C}^T  \bm{E}^{k+1} , \bm{\theta}^{k+1}-\bm{\theta}^\star \rangle \right] + \frac{\rho}{2} \mathbb{E}\left[\langle \bm{M}_{-} \bm{M}_{-}^T \bm{E}^{k+1} , \bm{\theta}^{k+1}-\bm{\theta}^\star \rangle \right].
\end{align}
\normalsize
Now, using identities \eqref{id4} and \eqref{id5}, we get the following bounds
\small
\begin{align}
\nonumber&\mathbb{E}\left[\langle \bm{C} \left(\bm{\theta}^{\star}- \bm{\theta}^{k+1}\right), \bm{\theta}^{k+1}-\bm{\theta}^\star \rangle \right] \\
&\leq \left(\frac{\eta_0}{2}\sigma_{\max}^2(\bm{C}) + \frac{1}{2 \eta_0}\right) \mathbb{E}\left[\|\bm{\theta}^{k+1}- \bm{\theta}^{\star}\|_F^2 \right], \label{eq_first}\\
\nonumber&\mathbb{E}\left[\langle \bm{C} \left(\bm{\theta}^{k}\!- \bm{\theta}^{\star}\right), \bm{\theta}^{k+1}\!-\!\bm{\theta}^\star \rangle \right] \\
&\leq \frac{\eta_1 \sigma_{\max}^2(\bm{C})}{2}  \mathbb{E}\left[\|\bm{\theta}^{k}\!-\! \bm{\theta}^{\star}\|_F^2 \right]\!+\!\frac{1}{2 \eta_1} \mathbb{E}\left[\|\bm{\theta}^{k+1}\!-\!\bm{\theta}^{\star}\|_F^2 \right], \\
\nonumber&\mathbb{E}\left[\langle \bm{\beta}^{k+1}-\bm{\beta}^{\star}, \bm{M}_{-}^T \bm{E}^{k+1} \rangle \right] \\
&\leq \frac{\eta_2}{2} \mathbb{E}\left[\|\bm{\beta}^{k+1}-\bm{\beta}^{\star}\|_F^2 \right] + \frac{\sigma_{\max}^2(\bm{M}_{-})}{2 \eta_2} \mathbb{E}\left[\|\bm{E}^{k+1}\|_F^2 \right], \\
\nonumber&\mathbb{E}\left[\langle \bm{C} \bm{E}^{k} , \bm{\theta}^\star-\bm{\theta}^{k+1} \rangle \right] \\
&\leq \frac{\eta_3}{2} \mathbb{E}\left[\|\bm{\theta}^{k+1}-\bm{\theta}^\star\|_F^2 \right] + \frac{\sigma_{\max}^2(\bm{C})}{2 \eta_3} \mathbb{E}\left[\|\bm{E}^{k}\|_F^2 \right], \\
\nonumber&\mathbb{E}\left[\langle \bm{C}^T  \bm{E}^{k+1} , \bm{\theta}^\star-\bm{\theta}^{k+1} \rangle \right] \\
&\leq \frac{\eta_4}{2} \mathbb{E}\left[\|\bm{\theta}^{k+1}-\bm{\theta}^\star\|_F^2 \right] + \frac{\sigma_{\max}^2(\bm{C})}{2 \eta_4} \mathbb{E}\left[\|\bm{E}^{k+1}\|_F^2 \right], \\
\nonumber&\mathbb{E}\left[\langle \bm{M}_{-} \bm{M}_{-}^T \bm{E}^{k+1} , \bm{\theta}^{k+1}-\bm{\theta}^\star \rangle \right] \\
&\leq \frac{\eta_5}{2} \mathbb{E}\left[\|\bm{\theta}^{k+1}-\bm{\theta}^\star\|_F^2 \right] + \frac{\sigma_{\max}^4(\bm{M}_{-})}{2 \eta_5} \mathbb{E}\left[\|\bm{E}^{k+1}\|_F^2 \right], \label{eq_last}
\end{align}
\normalsize
Replacing the bounds derived in \eqref{eq_first}-\eqref{eq_last} in \eqref{eqbound} and introducing $\kappa > 0$, we get
\small
\begin{align}
\nonumber & (1+\kappa) \frac{1}{\rho} \mathbb{E}\left[\|\bm{\beta}^{k+1}-\bm{\beta}^{\star}\|_F^2\right] +  \mu  \mathbb{E}\left[\|\bm{\theta}^{k+1}-\bm{\theta}^\star\|_F^2\right]\\
\nonumber& \leq \frac{1}{\rho} \mathbb{E}\left[\|\bm{\beta}^{k}-\bm{\beta}^{\star}\|_F^2\right] + \frac{\rho \sigma_{\max}^2(\bm{C})}{2} \eta_1  \mathbb{E}\left[\|\bm{\theta}^{k}-\bm{\theta}^\star\|_F^2\right]\\
\nonumber& + \rho \left(\frac{1}{2 \eta_1} + \frac{\eta_3}{2} + \frac{\eta_4}{2} + \frac{\eta_5}{4}\right)   \mathbb{E}\left[\|\bm{\theta}^{k+1}-\bm{\theta}^\star\|_F^2 \right]\\
\nonumber & + \left(\frac{\eta_2}{2} + \frac{\kappa}{\rho}\right) \mathbb{E}\left[\|\bm{\beta}^{k+1}-\bm{\beta}^{\star}\|_F^2\right]  \\
\nonumber&+\!\rho\!\left(\frac{\eta_0 \sigma_{\max}^2(\bm{C})}{2}\!+\!\frac{1}{2 \eta_0}\right)\!\mathbb{E}\left[\|\bm{\theta}^{k+1}\!-\!\bm{\theta}^{\star}\|_F^2\right]\!+\!\frac{\rho \sigma_{\max}^2(\bm{C})}{2 \eta_3}\!\mathbb{E}\left[\|\bm{E}^{k}\|_F^2\right]\\
&\!+\!\left( \frac{\sigma_{\max}^2(\bm{M}_{-})}{2 \eta_2}\!+\!\frac{\rho \sigma_{\max}^2(\bm{C})}{2 \eta_4} \!+\!\frac{\rho \sigma_{\max}^4(\bm{M}_{-})}{4 \eta_5} \right) \mathbb{E}\left[\|\bm{E}^{k+1}\|_F^2\right].
\end{align}
\normalsize
Using that $\|\bm{E}^{k+1}\|_F^2 \leq \|\bm{E}^{k}\|_F^2$, and re-arranging the terms, we can further write
\small
\begin{align}\label{eqb_01}
\nonumber& \frac{1}{\rho} \mathbb{E}\left[\|\bm{\beta}^{k}-\bm{\beta}^{\star}\|_F^2\right] - \frac{1+\kappa}{\rho} \mathbb{E}\left[\|\bm{\beta}^{k+1}-\bm{\beta}^{\star}\|_F^2\right]\\
\nonumber& + \frac{\rho \eta_1 \sigma_{\max}^2(\bm{C}) }{2} \mathbb{E}\left[\|\bm{\theta}^{k}-\bm{\theta}^\star\|_F^2\right] + \left(\frac{\eta_2}{2} + \frac{\kappa}{\rho}\right) \mathbb{E}\left[\|\bm{\beta}^{k+1}-\bm{\beta}^{\star}\|_F^2\right]\\
\nonumber&\!-\!\left[\mu\!-\!\left(\frac{\eta_0 \sigma_{\max}^2(\bm{C})}{2}\!+\!\frac{1}{2 \eta_0}\!+\!\frac{1}{2 \eta_1}\!+\!\frac{\eta_3}{2}\!+\!\frac{\eta_4}{2}\!+\! \frac{\eta_5}{4}\right) \rho \right]\\
&\times \mathbb{E}\left[\|\bm{\theta}^{k+1}\!-\!\bm{\theta}^\star\|_F^2\right] + \gamma \mathbb{E}\left[\|\bm{E}^{k}\|_F^2\right] \geq 0,
\end{align}
\normalsize
where $\gamma = \frac{\sigma_{\max}^2(\bm{M}_{-})}{2 \eta_2} + \frac{\rho}{2} \left(\frac{1}{\eta_4} + \frac{1}{\eta_3} \right)  \sigma_{\max}^2(\bm{C}) + \frac{\rho}{4 \eta_5} \sigma_{\max}^4(\bm{M}_{-}) > 0$. Fixing $\eta_2 = \frac{2 \kappa}{\rho}$, we get
\small
\begin{align}\label{eqb_1}
\nonumber& \frac{1}{\rho} \mathbb{E}\left[\|\bm{\beta}^{k}-\bm{\beta}^{\star}\|_F^2\right] - (1+\kappa)\frac{1}{\rho} \mathbb{E}\left[\|\bm{\beta}^{k+1}-\bm{\beta}^{\star}\|_F^2\right]\\
\nonumber& + \frac{\rho}{2} \eta_1 \sigma_{\max}^2(\bm{C}) \mathbb{E}\left[\|\bm{\theta}^{k}-\bm{\theta}^\star\|_F^2\right] + \frac{2 \kappa}{\rho} \mathbb{E}\left[\|\bm{\beta}^{k+1}-\bm{\beta}^{\star}\|_F^2\right]\\
\nonumber&\!-\!\left[\!\mu\!-\!\left(\frac{\eta_0\sigma_{\max}^2(\bm{C})}{2}\!+\!\frac{1}{2 \eta_0}\!+\!\frac{1}{2 \eta_1}\!+\! \frac{\eta_3}{2}\!+\!\frac{\eta_4}{2}\!+\!\frac{\eta_5}{4}\right) \rho \right]\!\\
& \times \mathbb{E}\left[\|\bm{\theta}^{k+1}\!-\!\bm{\theta}^\star\|_F^2\right] + \gamma \mathbb{E}\left[\|\bm{E}^{k}\|_F^2\right] \geq 0.
\end{align}
\normalsize
In order to bound the term $\mathbb{E}\left[\|\bm{\beta}^{k+1}-\bm{\beta}^\star\|_F^2\right]$ in the left hand side, we use \eqref{eqbet} to write
\small
\begin{align}
\nonumber & \mathbb{E}\left[\|\bm{M}_{-} (\bm{\beta}^{\star}-\bm{\beta}^{k+1})\|_F^2\right]\\
\nonumber &= \mathbb{E}\left[\|\nabla\!f(\bm{\theta}^{k+1})\!-\!\nabla f(\bm{\theta}^\star)\!+\!\rho \bm{C} \left(\bm{\theta}^{k+1}\!-\!\bm{\theta}^{k}\right)\right.\\
&+\left.\!\rho \bm{C} \bm{E}^{k}\!+\!\rho (\bm{C}^T\!-\!\frac{1}{2} \bm{M}_{-} \bm{M}_{-}^T) \bm{E}^{k+1}\|_F^2\right].
\end{align}
\normalsize
Using identity \eqref{id1}, we can further write
\small
\begin{align}
\nonumber& \mathbb{E}\left[\|\bm{M}_{-} (\bm{\beta}^{\star}-\bm{\beta}^{k+1})\|_F^2\right] \\
\nonumber &\leq 2 \mathbb{E}\left[\|\nabla f(\bm{\theta}^{k+1}) - \nabla f(\bm{\theta}^\star) + \rho \bm{C} \left(\bm{\theta}^{k+1}-\bm{\theta}^{k}\right)\|_F^2\right]\\
& + 2 \mathbb{E}\left[\|\rho \bm{C} \bm{E}^{k} + \rho \left(\bm{C}^T-\frac{1}{2} \bm{M}_{-} \bm{M}_{-}^T \right) \bm{E}^{k+1}\|_F^2\right].
\end{align}
\normalsize
Using \eqref{id7} for the first term and \eqref{id1} for the second term of the right hand side, we get 
\small
\begin{align}
\nonumber &\mathbb{E}\left[\|\bm{M}_{-} (\bm{\beta}^{k+1}-\bm{\beta}^{\star})\|_F^2\right] \\
\nonumber&\!\leq\!2 \eta \mathbb{E}\left[\|\nabla f(\bm{\theta}^{k+1})\!-\!\nabla f(\bm{\theta}^\star) \|_F^2\right]\!+\!\frac{2 \eta}{\eta-1} \mathbb{E}\left[\|\rho \bm{C} \left(\bm{\theta}^{k+1}\!-\!\bm{\theta}^{k}\right)\|_F^2\right] \\
&+ 4 \mathbb{E}\left[\|\rho \bm{C} \bm{E}^{k}\|_F^2\right] + 4 \mathbb{E}\left[\|\rho \left(\bm{C}^T-\frac{1}{2} \bm{M}_{-} \bm{M}_{-}^T \right) \bm{E}^{k+1}\|_F^2\right].
\end{align}
\normalsize
Now, we can write $\mathbb{E}\left[\|\bm{M}_{-} (\bm{\beta}^{k+1}-\bm{\beta}^{\star})\|_F^2\right] \geq \tilde{\sigma}_{\min}^2(\bm{M}_{-})~\mathbb{E}\left[\|\bm{\beta}^{k+1}-\bm{\beta}^{\star}\|_F^2\right]$ where $\tilde{\sigma}_{\min}(\bm{M}_{-})$ is the minimum non-zero singular value of $\bm{M}_{-}$ since both $\bm{\beta}^{k+1}$ and $\bm{\beta}^{\star}$ belong to the columns space of $\bm{M}_{-}$ and from Assumption 5, we have $\mathbb{E}\left[\|\nabla f(\bm{\theta}^{k+1})\!-\!\nabla f(\bm{\theta}^\star) \|_F\right] \leq L~\mathbb{E}\left[\|\bm{\theta}^{k+1}\!-\!\bm{\theta}^\star \|_F\right]$. Therefore, we get
\small
\begin{align}\label{eqb_2}
\nonumber &\mathbb{E}\left[\|\bm{\beta}^{k+1}-\bm{\beta}^{\star}\|_F^2\right] \\
\nonumber &\leq \frac{2 \eta}{\tilde{\sigma}_{\min}^2(\bm{M}_{-})} \left(L^2 + \frac{2 \rho^2}{\eta-1} \sigma_{\max}^2(\bm{C})\right) \mathbb{E}\left[\|\bm{\theta}^{k+1} - \bm{\theta}^\star \|_F^2\right]\\
\nonumber & + \frac{4 \eta \rho^2 \sigma_{\max}^2(\bm{C})}{(\eta-1) \tilde{\sigma}_{\min}^2(\bm{M}_{-})}  \mathbb{E}\left[\|\bm{\theta}^{k}-\bm{\theta}^{\star}\|_F^2\right] \\
&+ \frac{16 N \rho^2}{\tilde{\sigma}_{\min}^2(\bm{M}_{-})} \left( \sigma_{\max}^2(\bm{C}) + \sigma_{\max}^2\left(\bm{C}^T-\frac{1}{2} \bm{M}_{-} \bm{M}_{-}^T \right) \right) \psi^{2k},
\end{align}
\normalsize
where we have used that $\mathbb{E}\left[\|\bm{E}^{k+1}\|_F^2\right] \leq \mathbb{E}\left[\|\bm{E}^{k}\|_F^2\right] \leq 4 C_0^2 N \psi^{2k}$. Plugging the bound obtained from \eqref{eqb_2} in \eqref{eqb_1} we get
\small
\begin{align}\label{eq_beffinal}
\nonumber& \frac{1}{\rho} \mathbb{E}\left[\|\bm{\beta}^{k}-\bm{\beta}^{\star}\|_F^2\right] - (1+\kappa) \frac{1}{\rho} \mathbb{E}\left[\|\bm{\beta}^{k+1}-\bm{\beta}^{\star}\|_F^2\right]\\
\nonumber&  + \left(b_1  + a \kappa \right) \rho \mathbb{E}\left[\|\bm{\theta}^{k}-\bm{\theta}^\star\|_F^2\right] \\
& - \left(\mu - \frac{c \kappa}{\rho} - (b_2 + a \kappa) \rho \right) \mathbb{E}\left[\|\bm{\theta}^{k+1}-\bm{\theta}^\star\|_F^2 \right] + \nu \psi^{2k} \geq 0,
\end{align}
\normalsize
where $b_1 = \frac{\eta_1 \sigma_{\max}^2(\bm{C})}{2} $, $b_2 = \frac{\eta_0}{2}\sigma_{\max}^2(\bm{C}) + \frac{1}{2 \eta_0} + \frac{1}{2 \eta_1} + \frac{\eta_3}{2} + \frac{\eta_4}{2} + \frac{\eta_5}{4}$, $c = \frac{4 \eta  L^2}{\tilde{\sigma}_{\min}^2(\bm{M}_{-})}$, $a = \frac{8 \eta \sigma_{\max}^2(\bm{C})}{(\eta-1) \tilde{\sigma}_{\min}^2(\bm{M}_{-})} $, and $\nu = 4 N \gamma + \frac{32 N \rho  \kappa}{\tilde{\sigma}_{\min}^2(\bm{M}_{-})} \left( \sigma_{\max}^2(\bm{C}) + \sigma_{\max}^2\left(\bm{C}^T-\frac{1}{2} \bm{M}_{-} \bm{M}_{-}^T \right) \right)$.\\
To ensure that there is a decrease in the optimality gap, we need to determine, for which values of $\rho$, we have $c - b_2 \rho - a \rho^2 > 0$ and $
\mu - \frac{c \kappa}{\rho} - (b_2 + a \kappa) \rho \geq (1+\kappa) (b_1 + a \kappa) \rho > 0$. In other words, we need to look for $\rho$ such that
\small
\begin{align}\label{eq_rho}
- \left[ (b_2 + a \kappa) + (1+\kappa) (b_1 + a \kappa) \right] \rho^2 + \mu \rho - c \kappa \geq 0.
\end{align}
\normalsize
The discriminant of the quadratic equation is $
\Delta = \mu^2 - 4 c \kappa \left[ (b_2 + a \kappa) + (1+\kappa) (b_1 + a \kappa) \right]$. To ensure that we can find $\rho$ such that \eqref{eq_rho} is satisfied, we need to impose that $\Delta > 0$. Since $\Delta$ is a third order equation in $\kappa$, finding, for which values of $\kappa > 0$, $\Delta > 0$ is not straightforward. However, since when $\kappa \rightarrow 0$, $\Delta \rightarrow \mu^2 > 0$, and knowing that $\Delta$ is a decreasing function with $\Delta \rightarrow - \infty$ as $\kappa \rightarrow \infty$, then we deduce that there exits $\bar{\kappa} > 0$ such that for $0 < \kappa < \bar{\kappa}$, we have $\Delta > 0$. In the remainder, we consider $\kappa$ such that $0 < \kappa < \bar{\kappa}$. Thus, for $0 < \rho < \bar{\rho}$, \eqref{eq_rho} holds where $
\bar{\rho} = \frac{\mu + \sqrt{\Delta}}{(b_2 + a \kappa) + (1+\kappa) (b_1 + a \kappa)}$. Therefore, we can write
\small
\begin{align}
\nonumber& \frac{1}{\rho} \mathbb{E}\left[\|\bm{\beta}^{k}-\bm{\beta}^{\star}\|_F^2\right] - (1+\kappa) \frac{1}{\rho} \mathbb{E}\left[\|\bm{\beta}^{k+1}-\bm{\beta}^{\star}\|_F^2\right] \\
\nonumber & + \rho \left(b_1  + a \kappa \right)  \mathbb{E}\left[\|\bm{\theta}^{k}-\bm{\theta}^\star\|_F^2\right]\\
& - \rho  (1+\kappa) \left(b_1  + a \kappa \right) \mathbb{E}\left[\|\bm{\theta}^{k+1}-\bm{\theta}^\star\|_F^2\right] + \nu \psi^{2k} \geq 0.
\end{align}
\normalsize
Re-arranging the terms, we get
\small
\begin{align}
\nonumber& \frac{1}{\rho} \mathbb{E}\left[\|\bm{\beta}^{k+1}-\bm{\beta}^{\star}\|_F^2\right] +   \rho \left(b_1  + a \kappa \right) \mathbb{E}\left[\|\bm{\theta}^{k+1}-\bm{\theta}^\star\|_F^2\right] \\
\nonumber & \leq \frac{1}{1+\kappa} \left(\frac{1}{\rho} \mathbb{E}\left[\|\bm{\beta}^{k}-\bm{\beta}^{\star}\|_F^2\right] +   \rho  \left(b_1  + a \kappa \right)\mathbb{E}\left[\|\bm{\theta}^{k}-\bm{\theta}^\star\|_F^2\right] \right)\\
&+ \frac{\nu }{1+\kappa} \psi^{2k}.
\end{align}
\normalsize
Using this equation iteratively, we can write
\small
\begin{align}
\nonumber & \frac{1}{\rho} \mathbb{E}\left[\|\bm{\beta}^{k+1}-\bm{\beta}^{\star}\|_F^2\right] +  \rho \left(b_1  + a \kappa \right)  \mathbb{E}\left[\|\bm{\theta}^{k+1}-\bm{\theta}^\star\|_F^2\right] \\
\nonumber &\!\leq\!\left(\frac{1}{1+\kappa}\right)^{k+1}\!\left(\frac{1}{\rho} \mathbb{E}\left[\|\bm{\beta}^{0}\!-\!\bm{\beta}^{\star}\|_F^2\right]\!+\!\rho (b_1\!+\!a \kappa ) \mathbb{E}\left[\|\bm{\theta}^{0}\!-\!\bm{\theta}^\star\|_F^2\right] \right) \\
& + \nu \sum_{j=0}^{k} \left(\frac{1}{1+\kappa}\right)^{k-j+1} \psi^{2j}.
\end{align}
\normalsize
Defining $\delta_1 = \min\{(1+\kappa)^{-1}, \psi^2 \}$ and $\delta_2 = \max\{(1+\kappa)^{-1}, \psi^2\}$, we can further write
\small
\begin{align}
\nonumber& \frac{1}{\rho} \mathbb{E}\left[\|\bm{\beta}^{k+1}-\bm{\beta}^{\star}\|_F^2\right] +  \rho  \left(b_1  + a \kappa \right) \mathbb{E}\left[\|\bm{\theta}^{k+1}-\bm{\theta}^\star\|_F^2\right] \\
\nonumber & \overset{\mathrm{(a)}}{\leq} \left(\frac{1+\delta_2}{2}\right)^{k+1}\!\left(\frac{1}{\rho} \mathbb{E}\left[\|\bm{\beta}^{0}\!-\!\bm{\beta}^{\star}\|_F^2\right]\!+\!\rho \left(b_1\!+\!a \kappa \right)\!\mathbb{E}\left[\|\bm{\theta}^{0}\!-\!\bm{\theta}^\star\|_F^2\right] \right)\\
\nonumber & + \nu \sum_{j=0}^{k} \left(\frac{1+\delta_2}{2}\right)^{k-j+1} \delta_1^{j} \\
\nonumber & \overset{\mathrm{(b)}}{\leq} \left(\frac{1+\delta_2}{2}\right)^{k+1}\!\left(\frac{1}{\rho}\mathbb{E}\left[\|\bm{\beta}^{0}\!-\!\bm{\beta}^{\star}\|_F^2 \right]\!+\!\rho  \left(b_1\!+\!a \kappa \right)\!\mathbb{E}\left[\|\bm{\theta}^{0}\!-\!\bm{\theta}^\star\|_F^2 \right] \right. \\
&\left. +  \frac{\nu(1+\delta_2)}{1+\delta_2-2 \delta_1} \right),
\end{align}
\normalsize
where we have used in $\mathrm{(a)}$ the fact that $\delta_2 \leq (1+\delta_2)/2$ since $\kappa > 0$ and $\psi \in (0,1)$ and $(2 \delta_1)/(1+\delta_2) \in (0,1)$ in $\mathrm{(b)}$. Since $(1+\delta_2)/2 \in (0,1)$, then we deduce that the sequence $(\bm{\theta}^{k}, \bm{\beta}^{k})$ converges to $(\bm{\theta}^{\star}, \bm{\beta}^{\star})$ linearly with a rate $(1+\delta_2)/2$. 


\bibliographystyle{IEEEtran}

\end{document}